\documentclass[12pt]{article}
\usepackage{amsmath}
\usepackage{graphicx}
\usepackage{natbib}
\usepackage{url} 
\usepackage{tikz-qtree}
\usepackage{graphicx,psfrag,epsf}
\usepackage{enumerate}
\usepackage{natbib}
\usepackage{amsmath}
\usepackage{amsthm}
\usepackage{amssymb}
\def\argmin{\mathop{\rm argmin}}

\def\RR{\mathbb R}
\def\ZZ{\mathbb Z}
\usepackage{multirow}
\usepackage{floatflt}
\usepackage{caption}
\usepackage{color}
\usepackage{nameref,hyperref}
\usepackage{pifont}
\usepackage{mathtools}
\usepackage{csvsimple}
\usepackage{algorithm}
\usepackage{algpseudocode}

\newcommand{\bd}{\mathbf{d}}

\newcommand{\bl}{\mathbf{l}}

\newcommand{\bx}{\mathbf{x}}

\newcommand{\by}{\mathbf{y}}

\DeclareGraphicsExtensions{.jpg,.pdf,.mps,.png}
\def\argmin{\mathop{\rm argmin}}

\def\RR{\mathbb R}
\def\ZZ{\mathbb Z}

\newcommand{\bmu}{{\boldsymbol{\mu}}}

\newcommand{\N}{{\cal N}}

\graphicspath{{images/}{pics/}{}}
\newtheorem{proposition}{Proposition}
\newtheorem{theorem}{Theorem}
\newtheorem{corollary}{Corollary}
\newtheorem{assumption}{Assumption}
\newtheorem{lemma}{Lemma}

\newcommand{\blind}{0}

\begin{document}

\def\spacingset#1{\renewcommand{\baselinestretch}%
{#1}\small\normalsize} \spacingset{1}


\if0\blind
{
  \title{\bf Scalable Clustering: Large Scale Unsupervised Learning of Gaussian Mixture Models with Outliers}
  \author{Yijia Zhou
    \hspace{.2cm}, Kyle A. Gallivan\\
    Department of Mathematics, Florida State University\\
    and \\
    Adrian Barbu \\
    Department of Statistics, Florida State University}
  \maketitle
} \fi

\if1\blind
{
  \bigskip
  \bigskip
  \bigskip
  \begin{center}
    {\LARGE\bf Title}
\end{center}
  \medskip
} \fi

\bigskip
\begin{abstract}
Clustering is a widely used technique with a long and rich history in a variety of areas. However, most existing algorithms do not scale well to large datasets, or are missing theoretical guarantees of convergence. 
This paper introduces a provably robust clustering algorithm based on loss minimization that performs well on Gaussian mixture models with outliers. 
It provides theoretical guarantees that the algorithm obtains high accuracy with high probability under certain assumptions. 
Moreover, it can also be used as an initialization strategy for $k$-means clustering. 
Experiments on real-world large-scale datasets demonstrate the effectiveness of the algorithm when clustering a large number of clusters, and a $k$-means algorithm initialized by the algorithm outperforms many of the classic clustering methods in both speed and accuracy, while scaling well to large datasets such as ImageNet.
\end{abstract}

\noindent%

\spacingset{1.45}
\section{Introduction}
\label{sec:intro}
Clustering is an important unsupervised learning technique with applications in many areas including information retrieval \citep{jardine1971use}, image segmentation \citep{coleman1979image}, pattern recognition \citep{diday1981clustering}, data mining \citep{mirkin2005clustering}, disease diagnosis \citep{alashwal2019application}, and more.  
One of the most commonly used non-probabilistic clustering approaches is the $k$-means algorithm \citep{lloyd1982least}. 
Probabilistic clustering models, which can be characterized as more sophisticated versions of $k$-means, are based on Gaussian Mixture Models (GMM) and yield more flexibility than $k$-means. 
With GMMs, it is assumed that the data points are Gaussian distributed; this is a less restrictive assumption than saying they are circular around a mean. In this way, a mean vector and the covariance matrix can be used to describe a cluster.

The motivation for this work comes from the problem of object recognition from images.  
An image usually contains one or more regions/objects of interest and the rest is meaningless background. 
This paper introduces a Gaussian mixture model with outliers where the Gaussian mixture components represent the objects of interest (positives), and the outliers (negatives) represent the background images that do not cluster together.
However, since data is usually standardized to zero mean and standard deviation one, it is assumed that the outliers come from a zero mean Gaussian distribution.

In this paper, we are only interested in scalable clustering - methods that scale well to data with millions of observations, thousands of dimensions and a large number (e.g. thousands) of clusters. 
Furthermore, the proposed algorithm has theoretical guarantees of convergence of the estimated clusters to the true cluster labels.

The main contributions of this paper are summarized as follows:
\begin{enumerate}
    \item The Gaussian mixture model with outliers is introduced as a simple framework for image classification problems in computer vision.
    \item A novel clustering algorithm, Scalable Clustering by Robust Loss Minimization (SCRLM), is developed for the model.
    The basic idea of SCRLM is to find the positive clusters as local minima of a robust loss function that is non-zero only within a certain radius from the cluster centers and zero everywhere else, and extract the clusters one-by-one.
    \item Theoretical guarantees are given that SCLRM is able to correctly cluster all the inliers and detect all the outliers with high probability under certain assumptions.
    \item The performance predictions are validated with experiments using simulated data and real data. The simulation results indicate that when the assumptions are met, SCRLM outperforms other algorithms such as $k$-means++, EM, t-SNE and spectral clustering.
    \item Experiments on real data indicate that SCRLM is very effective when the number of clusters and the data dimension are large, and it can be used as an initialization method for $k$-means clustering, outperforming $k$-means++ in accuracy and computation time.
\end{enumerate}
The rest of the paper is organized as follows. In Section \ref{sec:literature},  an overview of the literature on various existing clustering methods are given. Section \ref{sec:scrlm} will develop a novel algorithm, SCRLM, to solve the
clustering problem in the Gaussian mixture model with outliers, and the theoretical guarantees will be derived. In Section \ref{sec:exp},
experiments on both synthetic and real data will show that the proposed algorithm is scalable, efficient and accurate. Section \ref{sec:conc} summarizes the findings and concludes the paper with a
discussion of future research work.

\section{Literature Reviews}
\label{sec:literature}

The study of Gaussian mixture models can be traced back to \cite{pearson1894contributions}. 
The idea of using Gaussian mixtures in unsupervised learning was popularized by \cite{duda1973pattern}. 
The Expectation Maximization (EM) algorithm \citep{dempster1977maximum} was one of the first clustering algorithms for GMM.
\citet{xu1996convergence} analyzed the convergence of EM for well-separated Gaussian mixtures. \citet{dasgupta2007probabilistic} proposed a two-round variant of the EM algorithm and showed that, with high probability, it can recover the parameters of the Gaussians to near-optimal precision. 
In recent years, approaches have been proposed to improve convergence guarantees and applied to different kinds of GMMs.
\citet{dwivedi2018theoretical} provided theoretical guarantees in two classes of misspecified mixture models and \citet{segol2021improved} improved sample size requirements for accurate estimation by EM and gradient EM.

Tensor Decomposition, a spectral decomposition technique, also played an important role in learning GMMs. 
\citet{hsu2013learning}  developed a method based on moments with up to third order and provided theoretical guarantees that non-degenerate mixtures of spherical gaussians can be learned in polynomial time without any separation condition.

The algorithms above, designed for GMMs, are categorized as distribution-based clustering. Aside from those methods, there are other clustering methods that do not use statistical distributions to cluster the data objects.


Hierarchical clustering \citep{johnson1967hierarchical},  also known as connectivity-based clustering, creates a complete dendrogram  of the data. 
It is either agglomerative (bottom-up) or divisive (top-down). 
In agglomerative hierarchical clustering, the similarities between clusters are measured by distances between points, which is referred as linkage. In complete linkage, the distance between the farthest points are taken as the intra cluster distance which is less susceptible to outliers than single linkage \citep{tan2016introduction}. The main disadvantage of hierarchical clustering is, due to high time and space complexity, it is not suitable for large-scale datasets.

In contrast to hierarchical clustering, $k$-means \citep{lloyd1982least} is one of the  most famous centroid-based clustering  algorithms, which works by minimizing the squared distances between every point and its nearest centroid. 
Since $k$-means is an iterative algorithm involving initialization, clustering and centroids
updates, proper initialization techniques such as Maxmin \citep{gonzalez1985clustering}, Refine \citep{bradley1998refining} and $k$-means++ \citep{arthur2006k} have been proposed to improve the clustering results. 
\citet{franti2019much} demonstrated that for well-separated clusters, the performance of $k$-means depends completely on the goodness of initialization and $k$-means++ is the best one among those methods.
 
Spectral clustering  \citep{donath1973lower,shi2000normalized,meilua2001random,ng2002spectral,von2007tutorial} is a graph-based clustering algorithm that utilizes the eigenvectors of the adjacency matrix for dimension reduction. It is simple to implement but computationally expensive unless the graph is sparse and the similarity matrix can be efficiently constructed. \citet{vempala2004spectral} investigated the theoretical performance of spectral clustering in the isotropic Gaussian mixture model and proved that with high probability, exact recovery of the underlying cluster structure was achieved under a strong separation condition. 
\citet{loffler2021optimality} showed that spectral clustering is minimax optimal in Gaussian mixture models with isotropic covariance, when the number of clusters is fixed and the signal-to-noise ratio is large enough.

$t$-distributed stochastic neighbor embedding (t-SNE) \citep{van2008visualizing} is a technique that visualizes high-dimensional data that is usually processed before clustering. 
$k$-means++ and other clustering algorithms can be applied to the low-dimensional feature space obtained from t-SNE.

In summary, $k$-means++ has high scalability  with theoretical guarantees but it does not perform well in high dimension and is very sensitive to outliers. 
Spectral clustering and tensor decomposition are not sensitive to outliers but do not perform well on large-scale and high dimensional data. 
The proposed SCRLM method is a novel algorithm with efficiency and  strong theoretical guarantees  in dealing with outliers, scalability, and is suitable for high dimensional data.

\section{Scalable Clustering by Robust Loss Minimization}
\label{sec:scrlm}

Given a set $X = \{\bx_i \in \RR^p, i =1, ...,N\}$ of $N$ points from a Gaussian mixture model with outliers containing $m$ Gaussians, the goal is to group these points into $m$ compact subsets. 
In this paper, we only deal with large and high-dimensional data and a reasonably large number of clusters, specifically, $N \approx 10^6, p \approx 10^3, m \approx 10^3$.

A Gaussian mixture model with outliers is a weighted sum of $m$ component Gaussian densities and outliers as given by the equation,
\begin{equation}\label{equ:GMM}
p(\mathbf{x} \mid \Theta)=\sum_{i=1}^{m} w_{i} \N\left(\mathbf{x} \mid \boldsymbol{\mu}_{i}, \Sigma_{i}\right)
+w_{-1} O(\bx)
\end{equation}
where $\mathbf{x}\in \RR^p$ is a $p$-dimensional data point, $w_{i}, i\in \{-1, 1,2, \ldots, m\}$, are the mixture weights, $\N\left(\mathbf{x} \mid \boldsymbol{\mu}_{i}, \Sigma_{i} \right), i=1, \ldots, m$, are the component Gaussian densities and $O(\bx)$ is the distribution of the outliers. 

It is assumed that each component density is a $p$-variate isotropic Gaussian function of the form,
\begin{equation}
\N\left(\mathbf{x} \mid \boldsymbol{\mu}_{i}, \Sigma_{i} \right)=\N\left(\mathbf{x} \mid \boldsymbol{\mu}_{i}, \sigma_{i}^2\right)=\frac{1}{(2 \pi)^{p / 2}\left|\sigma_{i}^2I_p\right|^{1 / 2}} \exp \left\{-\frac{1}{2}\left(\mathbf{x}-\boldsymbol{\mu}_{i}\right)^{T} (\sigma_{i}^2I_p)^{-1}\left(\mathbf{x}-\boldsymbol{\mu}_{i}\right)\right\}
\end{equation}
with mean vector $\boldsymbol{\mu}_{i}$ and covariance matrix $\Sigma_i=\sigma_{i}^2I_p$.

Let  $l(\bx)\in \{-1, 1,2, \ldots, m\}$ be the label of observation $\bx$, i.e. the mixture component from which it was generated.
The samples $\bx_i$ with $l(\bx_i)>0$ are called positives and the outliers (with $l(\bx_i)=-1$) are also called negatives.

Inspired by real data examples, where the observations are standardized feature vectors generated by a convolutional neural network (CNN) from real images of certain objects or background, in this paper, it is assumed that the $m$ centers $\bmu_i, i =1, ...m$ and all outliers are generated from $O(\bx)=\mathcal{N}\left(\boldsymbol{0},I_{p}\right)$ and the label $i$ positives are generated with frequency $w_i$ from $\mathcal{N}(\bmu_i,\sigma_i^2)$, where $\sigma_i<1, \forall i$.
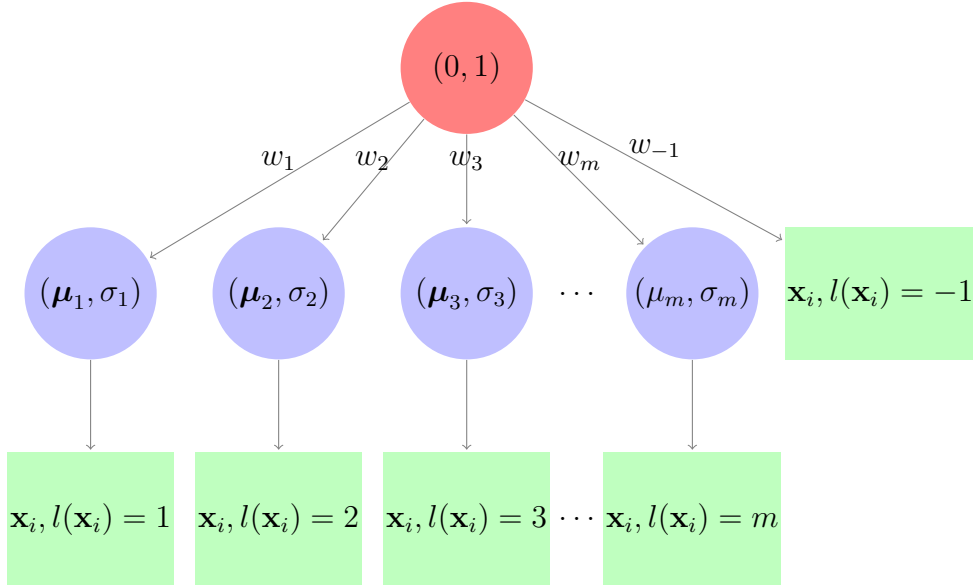
\begin{figure}[t]
        \begin{center}
    \begin{tikzpicture}[
   shorten >=1pt,->,
   draw=black!50,
    node distance=\layersep,
    every pin edge/.style={<-,shorten <=1pt},
    positive/.style={circle,fill=blue!25,minimum size=50pt,inner sep=1pt},
    obs/.style={rectangle,fill=green!25,minimum size=50pt,inner sep=1pt},
    level-0/.style={positive, fill=red!50},
    level1/.style={positive},
    outliers/.style={obs},
    annot/.style={text width=8em, text centered}
]
    \node[level-0] (0) at (5,0) {$(0,1)$};  
    \node[level1] (H1-1) at (0,-3) {($\bmu_1,\sigma_1$)};
    \node[level1] (H1-2) at (2.5,-3) {($\bmu_2,\sigma_2$)};
    \node[level1] (H1-3) at (5,-3) {($\bmu_3,\sigma_3$)};
    \node at (6.5,-3) {$\dots$};
    \node at (6.5,-6) {$\dots$};
    \node[level1] (H1-4) at (8,-3) {($\mu_m,\sigma_m$)};
     \node[outliers] (H1-5) at (10.5,-3) {$\bx_i,l(\bx_i)=-1$};
      \node[obs] (O1) at (0,-6) {$\bx_i,l(\bx_i)=1$};
    \node[obs] (O2) at (2.5,-6) {$\bx_i,l(\bx_i)=2$};
     \node[obs] (O3) at (5,-6) {$\bx_i,l(\bx_i)=3$};
     \node[obs] (O4) at (8,-6) {$\bx_i,l(\bx_i)=m$};
     \path (0) edge
     node[above]{$w_1$} (H1-1) ;
    \path (0) edge node[above]{$w_2$} (H1-2);
     \path (0) edge
     node[above]{$w_3$} (H1-3);
      \path (0) edge  node[above]{$w_m$} (H1-4);
      \path (0) edge
    node[above]{$w_{-1}$} (H1-5) ;
   \path (H1-1) edge  (O1);
 \path (H1-2) edge  (O2);
  \path (H1-3) edge  (O3);
   \path (H1-4) edge  (O4);
\end{tikzpicture}
\caption{Structure of the Gaussian mixture model with outliers used in this paper.}
\label{fig:structure of GMM}
\end{center}
    \end{figure}
 The structure of the Gaussian mixture model with outliers used in this paper is shown in Figure \ref{fig:structure of GMM}.

The problem of interest is to cluster a set of unlabeled observations generated from such a Gaussian mixture model with outliers and recover the labels $l(\bx_i)$ as well as the mixture parameters $w_i,\bmu_i, \sigma_i$.

\subsection{Robust Loss Function}

A loss minimization approach is taken, using the following loss function,

\begin{equation}\label{eq:lossfn}
L(\bx; \rho)=\sum_{i=1}^{N} \ell\left(\mathbf{x}_{i}-\bx; \rho\right)= \sum_{i=1}^{N} \min \left(\frac{\|\bx_i-\bx\|^{2}}{p \rho^{2}}-F, 0\right)
\end{equation}
where the per-observation loss, illustrated in Figure \ref{fig:robust_loss_function} is,
\begin{equation}
\ell(\bd; \rho)=\min \left(\frac{\|\bd\|^{2}}{p \rho^{2}}-F, 0\right).
\label{eq:loss1}
\end{equation}
The loss function $\ell(\bd; \rho)$ is zero outside a ball of radius $R_{\rho}=\rho \sqrt{p F}$.
The constant $F$  is set in this paper to $F=2.5$.

\begin{figure}[ht]
\centering
\begin{tabular}{cc}
 \includegraphics[width=0.45\linewidth]{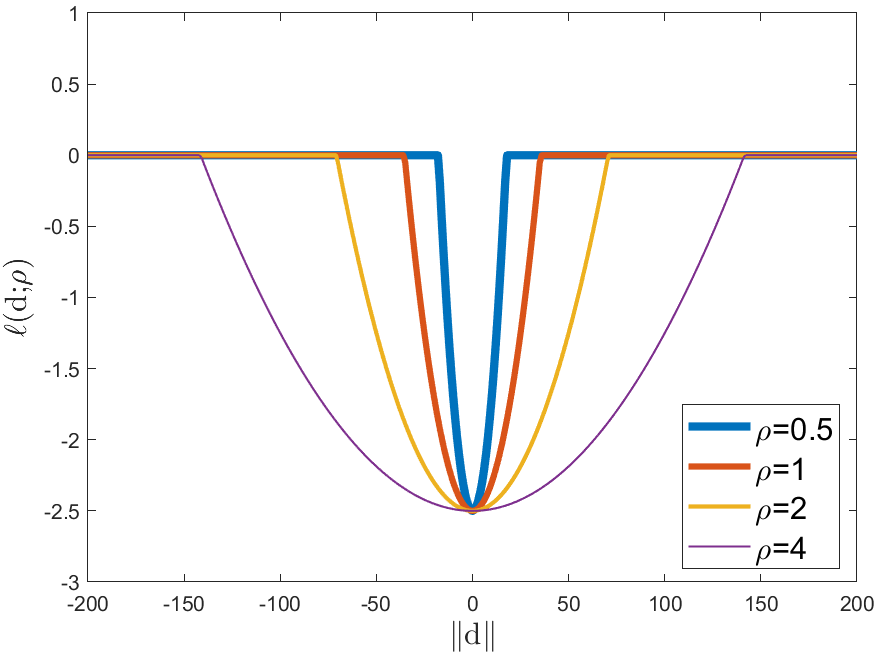}    
 & \includegraphics[width=0.45\linewidth]{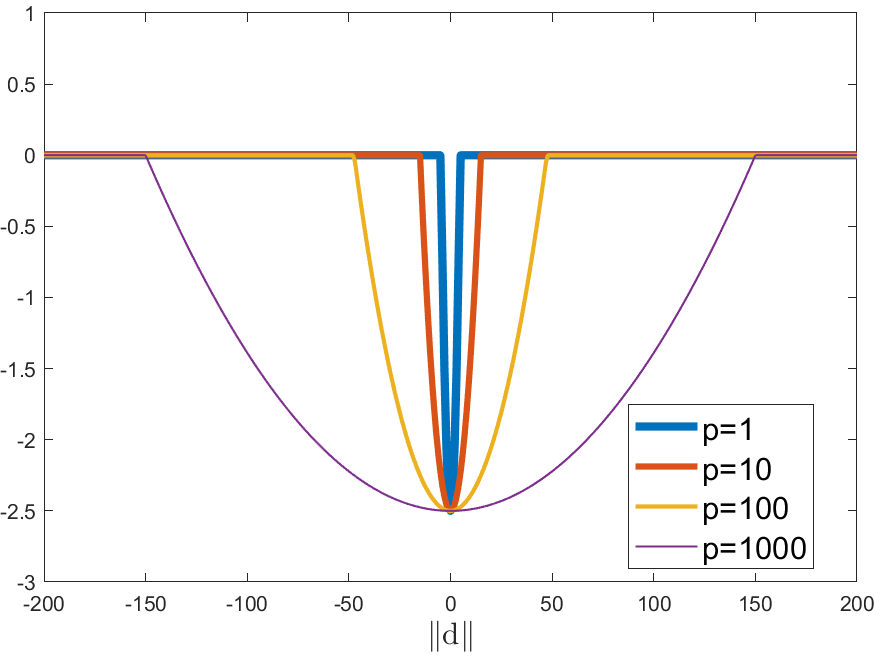} \\
 a) $p = 500$ & b) $\rho = 3$ \\
\end{tabular}
\vskip -3mm
\caption{The robust loss function $\ell(\bd; \rho)$ for different values of $p$ and $\rho$. \label{fig:robust_loss_function}} 
\end{figure}

The idea of the algorithm is to find the cluster centers $\boldsymbol{\mu}$ as local minima of the loss function \eqref{eq:lossfn}. For computational reasons, the centers $\bmu$ are sought among a subsample $S\subset \{1,...,N\}$ of the observations $\bx_i,i=1,...,N$,
\begin{equation}
    \bmu=\bx_k, \text { where }k =\argmin \limits_{i\in S} L\left(\boldsymbol{\bx_i}; \rho\right).\label{eq:findmu}
\end{equation}

After one cluster center $\bmu=\bx_k$ has been found, all samples from $S$ within the radius $\rho\sqrt{pF}$ from the $\bmu$ are considered as belonging to this cluster and are removed. The process is repeated until $\min\limits_{j \in S} L\left(\boldsymbol{\bx_j}; \rho\right)=-F$. 
\begin{figure}[ht]
\centering
\begin{tabular}{cc}
 \includegraphics[width=0.4\linewidth]{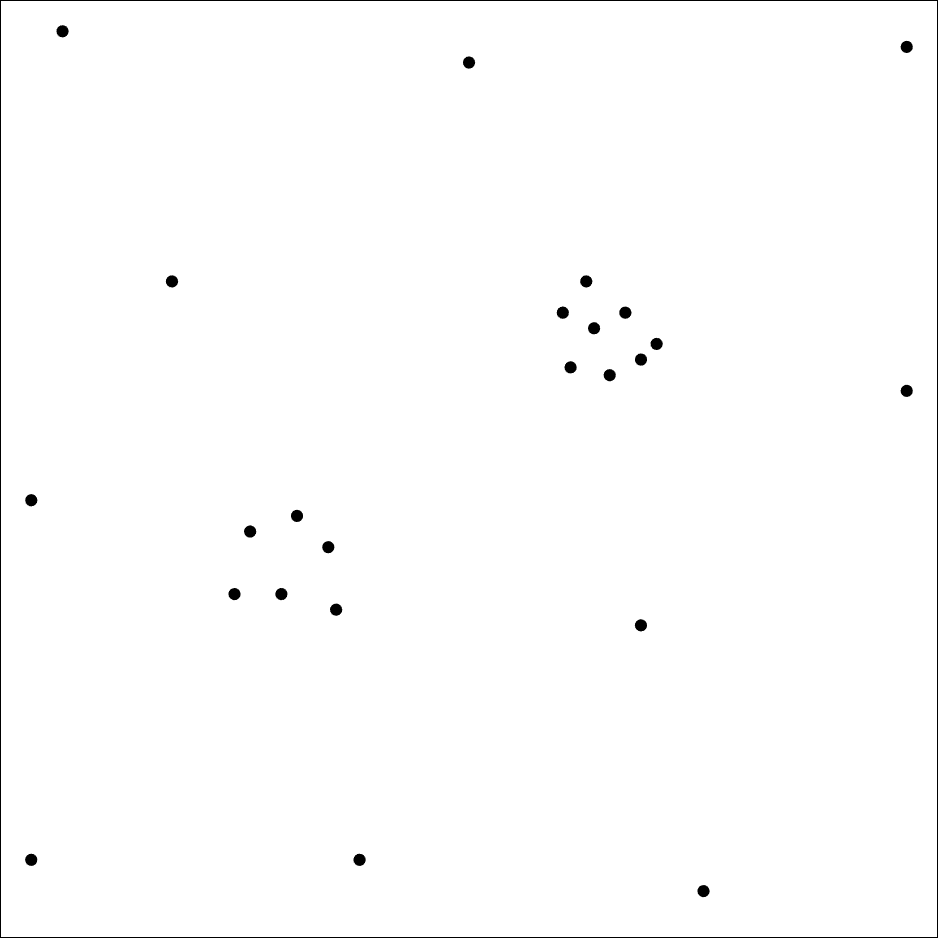}  
 & \includegraphics[width=0.4\linewidth]{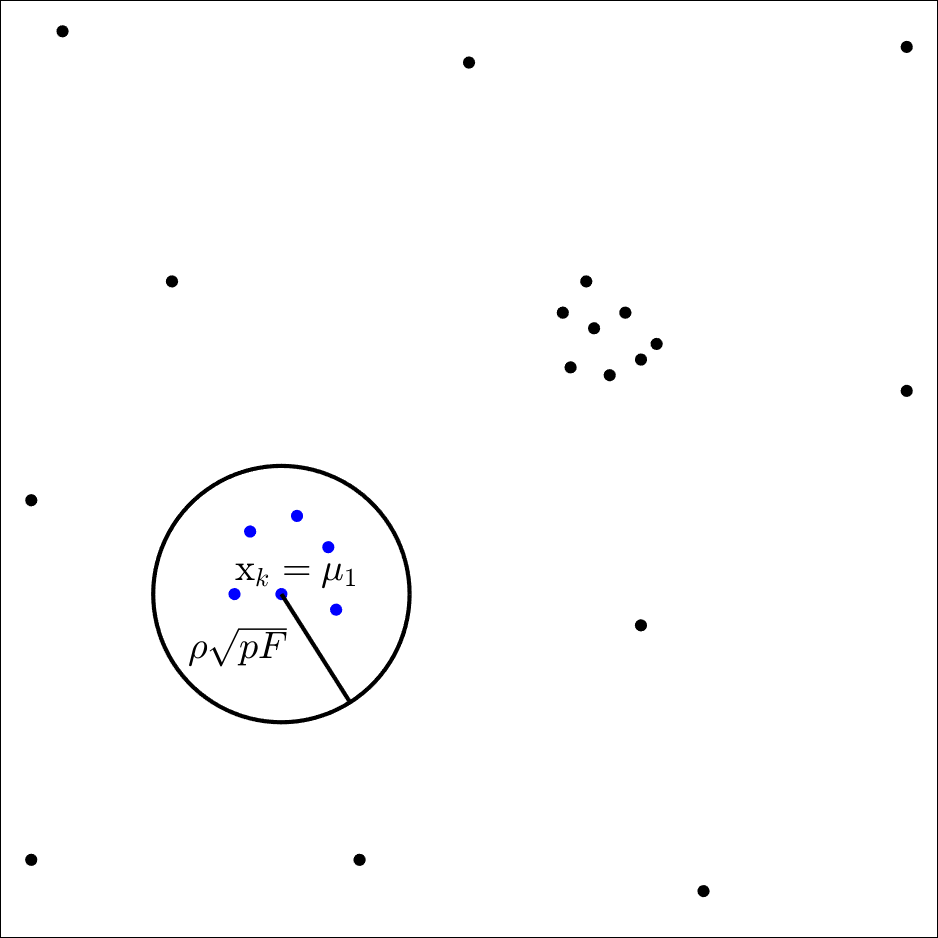} \\
 a) & b)\\
\includegraphics[width=0.4\linewidth]{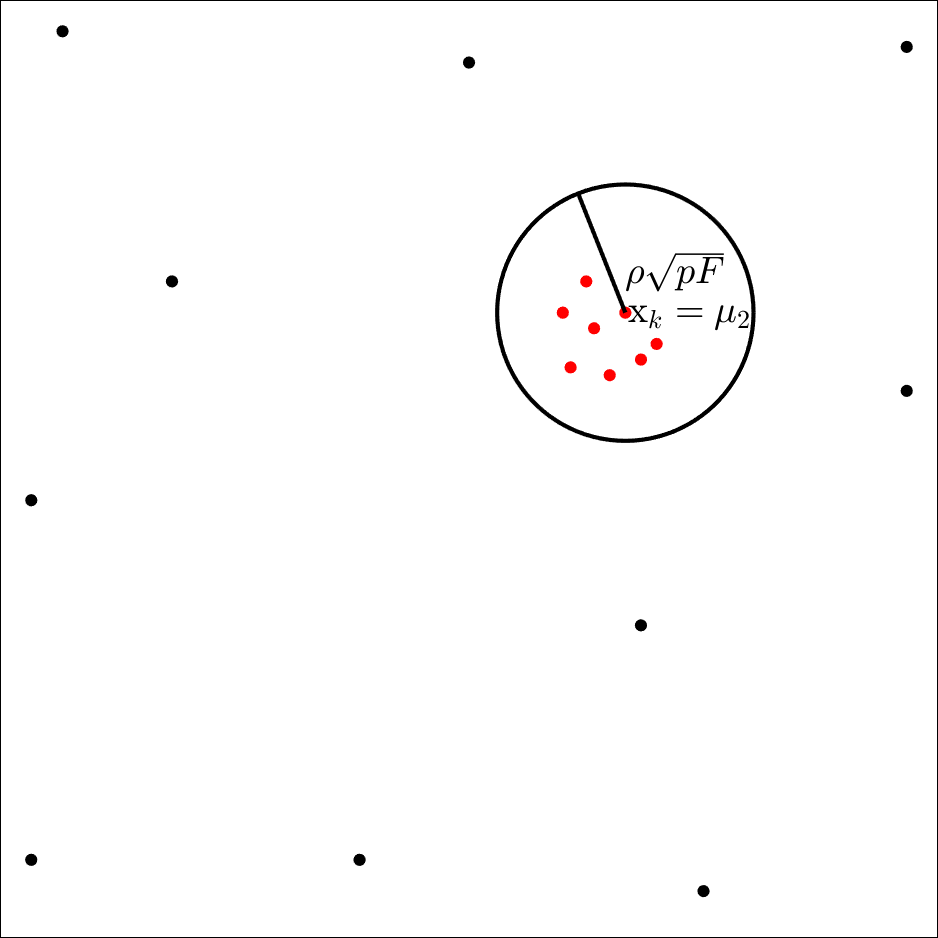}     
&\includegraphics[width=0.4\linewidth]{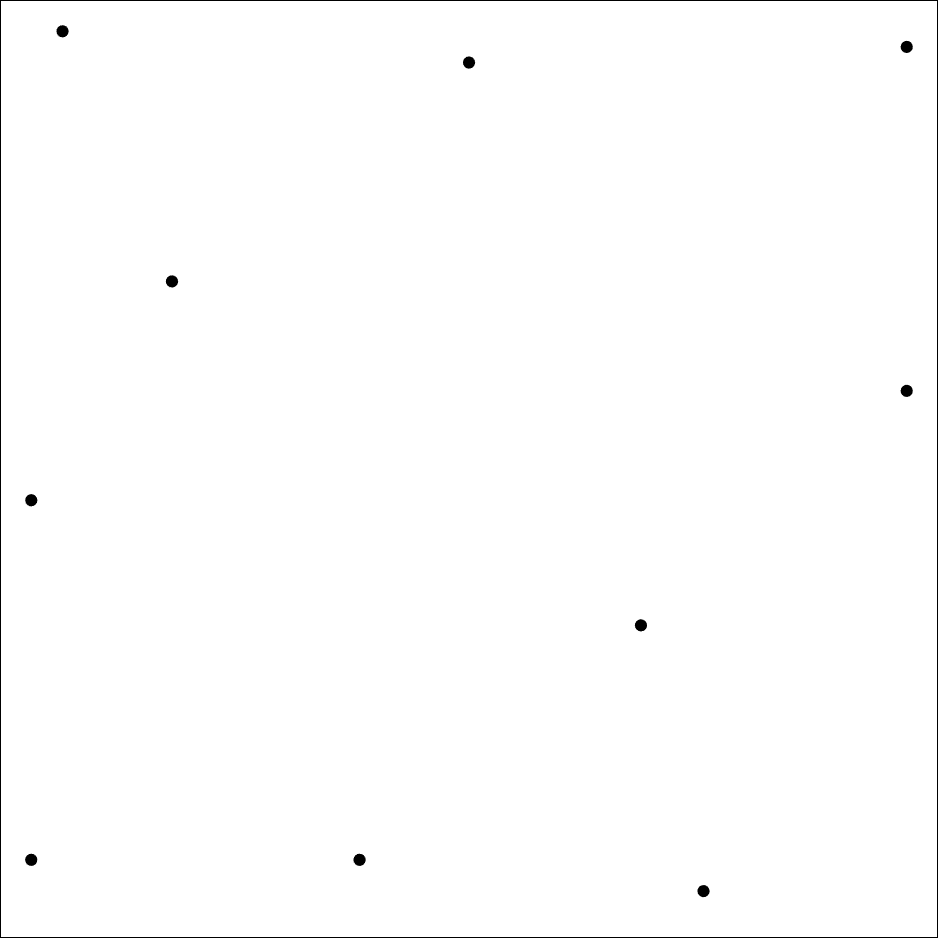}\\
c) & d)
\end{tabular}
\vskip -3mm
\caption{Diagram illustrating Algorithm \ref{alg:scRLM}. \label{fig:find}} 
\end{figure}

The process is illustrated in Figure \ref{fig:find}. Suppose there are two clusters and some outliers, as shown in Figure \ref{fig:find}a). 
The first cluster center is found as the sample with minimum loss \eqref{eq:lossfn} among all subsamples.
Then, all subsamples within the given radius $R=\rho\sqrt{p F}$ to this center are identified as belongings to this cluster (Figure \ref{fig:find}b)) and are removed.
Then the process is repeated to find another cluster center (Figure \ref{fig:find}c)) and all subsamples within the same radius are removed. 
Now all the subsamples have loss values $-F$, which means there are no clusters left, so the remaining subsamples are negatives (Figure \ref{fig:find}d)).

Once all the cluster centers are found, the label of each observation is assigned to its nearest center based on the norm distances. If the nearest distance is greater than $\rho \sqrt{pF}$, it is classified as an outlier. The procedure is summarized in Algorithm \ref{alg:scRLM}.
\begin{algorithm}[!ht]
\caption{Scalable Clustering by Robust Loss Minimization (SCRLM)}
		\label{alg:scRLM}
		\begin{algorithmic}[1]
			\State {{\bf Input:} $X=\{\boldsymbol{\bx_1},...,\boldsymbol{\bx_N}\}\subset \RR^p$, the bandwidth parameter $\rho$, the desired number of cluster $T$.}
		\State{{\bf Output:} the number of clusters $m$, the centers $\mu_t$ for $t=1,...,m$, cluster labels $l_1,...,l_N\in \{-1,1,2,...,m\}$.}
		\State{Randomly select a set of $n$ indices $S=\{i_1,...,i_n\}\subset \{1,...,N\}$ without replacement from $\{1,...,N\}$.}
			\State Compute $L(\boldsymbol{\bx_j}; \rho)=\sum\limits_{i=1}^{N} \ell\left(\mathbf{x}_{i}-\boldsymbol{\bx_j}; \rho\right)$
			$\forall j \in S$, where $\ell(\bd; \rho)=\operatorname{min} \left(\frac{\|\bd\|^{2}}{p \rho^{2}}-F, 0\right)$.
			\For{$t$ = 1 to $T$}
			\State Find $ k=\argmin\limits_{j \in S} L\left(\boldsymbol{\bx_j}; \rho\right)$
			\If{$L\left(\boldsymbol{\bx_k}; \rho\right)<-F$}
			\State Obtain one positive cluster center as $\bmu_t=\bx_k$
			\State Update $S\leftarrow S-\left\{i \in S,\left\|\boldsymbol{\bx_i}-\bmu_t\right\|<\rho\sqrt{p F}\right\}$
			\Else
			\State break
			\EndIf
			\EndFor
    
	\For{$i$ = 1 to $N$}
	    \State Compute $k=\argmin\limits_{j} \|\bx_i-\bmu_j\|$
	    \If{$\|\bx_i-\bmu_k\|<\rho\sqrt{p F}$}
		\State $l_i=k$
		\Else
		\State $l_i=-1$
		\EndIf
    \EndFor
		\end{algorithmic}
\end{algorithm}

\subsection{Theoretical Guarantees}
First, we summarize the notations used in this paper and the main assumption used in the derivation of the main result.
\begin{itemize}
\item $N$: the total number of observations
\item $S$: a subset of all the observations
\item $n$: the cardinality of $S$, $n=|S|$
\item $p$: the dimension of the observations $\bx_{i} \in \mathbb{R}^{p}$
\item $l(\bx)$: the true label (cluster assignment) of observation $\bx$
\item $m$: the true number of positives clusters 
\item $T$: the number of iterations (maximum number of clusters desired) in Algorithm \ref{alg:scRLM}
\item $S_{k}$: the elements of $S$ with label $k$, $S=\{\bx \in S, l(\bx)=k\}$ 
\item $w_{k}$: the weight of positive cluster $k$, $k=\overline{1,m}$
\item $w_{-1}$: the weight of negatives (observations $\bx$ with $l(\bx)=-1$)
\item $\bmu_{k}, \sigma_{k}$: true mean and standard deviation of positive cluster $k$
\item $\sigma_{max}=\max_{k\geq 0} \sigma_k$: the maximum standard deviation among all positive clusters
\item $H$: the set of all the negatives
\item $F$: a constant in the loss function \eqref{eq:lossfn}, in this paper, $F=2.5$
\item $\rho$: the bandwidth parameter in the loss function \eqref{eq:lossfn}
\item $R_{\rho}=\rho \sqrt{p F}$: the radius of the support of the loss function \eqref{eq:lossfn}
\item $L(\bx;\rho)$: the loss function \eqref{eq:lossfn}
\item $\ell(\bx;\rho)$: the per-observation loss function \eqref{eq:loss1} 
\end{itemize}
The following is the main assumption that is needed for the theoretical guarantees.
\begin{assumption}
\label{assumption1}
$\sigma_{max} \leq \rho<\sqrt{0.6}$, where $\rho$ is the bandwidth parameter for the loss function \eqref{eq:lossfn}.
\end{assumption}
Then, we obtain the following theorem guaranteeing that Algorithm \ref{alg:scRLM} (SCLRM) can detect all outliers and cluster all positives correctly with high probability.
\begin{theorem}\label{thm1}
Given $N$ samples from a GMM with outliers, with $w_i\geq a/m, i=\overline{1,m}$ for some $a>0$ and $\sigma_{max}\leq \rho < \sqrt{0.6}$, then Algorithm \ref{alg:scRLM} (SCLRM) using $|S|=n$ subsamples has $100\%$ accuracy with probability at least  $$1-10N^2\exp\{ -p/128\}-m\exp\{-na/m\}-2m\exp\{ -p/128\}-m\exp\{-a(N-1)/m\}$$
\end{theorem}
The proof of this theorem is given in Appendix \ref{sec:D}.

Based on Theorem \ref{thm1}, we have the following Corollary \ref{cor7} on the theoretical bounds for parameters $p$, $n$ and $N$.
\begin{corollary}\label{cor7}
Given $N$ samples from a GMM with outliers, with $w_i\geq a/m, i=\overline{1,m}$ for some $a>0$ and $\sigma_{max}\leq \rho < \sqrt{0.6}$, for any $\delta>0$,  if 
\[
p >128  (2\log {N} + \log \frac{40}{\delta}),
\]
\[
p >128  (\log m +\log \frac{8}{\delta}),
\]
\[
n>\frac{m}{a}(\log m+\log \frac{4}{\delta})
,\]
\[
N>\frac{m}{a}(\log m+\log \frac{4}{\delta})+1,
\]
then Algorithm \ref{alg:scRLM} (SCRLM) using $|S|=n$ subsamples will have $100\%$ accuracy with probability at least  $1-\delta$. 
\end{corollary}
The proof of this corollary is given in Appendix \ref{sec:D}.

\subsection{Computational Complexity}
\label{sec:time complexity}
Computing $L(\bx_j;\rho),j\in S$ (Step 4 of Algorithm \ref{alg:scRLM}) is $O(nNp)$.
Each iteration of steps 6-12 is $O(np)$, so steps 5-13 take $O(nmp)$.
Similarly, steps 14-21 take $O(Nmp)$.
Therefore, the computation complexity of Algorithm \ref{alg:scRLM} is $O(nNp+nmp+Nmp)=O(nNp+Nmp)$. 
From Corollary \ref{cor7} one could see that the subsample size $n$ should be chosen on the order of $O(m\log m)$. 
Therefore, the computational complexity of Algorithm \ref{alg:scRLM} is $O(mpN\log m)$, it is linear in the dimension $p$ and the number of observations $N$ and log-linear in the number of clusters $m$.

\section{Experiments}
\label{sec:exp}

This section presents an empirical evaluation of the performance of SCLRM using synthetic data and real datasets from computer vision. 
First, the tightness of the parameter bounds given in the theoretical guarantees are evaluated using synthetic data. 
Then, the effectiveness of SCRLM in real applications is evaluated using five real image datasets.

To compare the performances of SCRLM on synthetic and real datasets, two evaluation measures are defined for a true labeling vector $\bl\in \ZZ^N$ and an obtained labeling vector $\hat{\bl}\in \ZZ^N$: 
\begin{enumerate}
    \item $\operatorname{accuracy}(\bl, \hat{\bl})= \frac{1}{N} \max _{\pi \in P} 
    |\pi(\hat{\bl}) \cap \bl
    |$
    \item $\operatorname{purity}(\bl, \hat{\bl})=\frac{1}{N} \sum_{i=1}^{T} \max _{j}   |\hat{\bl}^{-1}(i) \cap \bl^{-1}(j) |$
\end{enumerate}
where $P$ is the set of all permutations of $\{1,...,m\}$, and 
$\bl^{-1}(j)=\{i, \bl_i=j\}$. The accuracy is computed in polynomial time using the Hungarian algorithm.

 
In order to assess the effectiveness of SCLRM, its performance is compared with the following clustering methods: $k$-means++ \citep{arthur2006k}, Complete Linkage Clustering  (CL) \citep{johnson1967hierarchical}, Spectral Clustering (SC) \citep{ng2002spectral}, Tensor Decomposition (TD) \citep{hsu2013learning}, Expectation Maximization (EM) \citep{dempster1977maximum} and t-Distributed Stochastic Neighbor Embedding (t-SNE) \citep{van2008visualizing}.

For consistency in comparing the accuracy and running time, experiments use an implementations of SCLRM and the state-of-the-art algorithms in MATLAB. 
For $k$-means++, the built in function \texttt{kmeans} that implements $k$-means++ is used. 
The built-in function \texttt{linkage} and \texttt{spectralcluster} are used for CL  and SC respectively. 
In terms of EM, a standard EM for GMM is used. 
In terms of TD, Theorem 2 from \citep{hsu2013learning} has been implemented in MATLAB. 
For t-SNE+$k$-means++, the built-in \texttt{tsne} function is used to generate a matrix of two-dimensional embeddings followed by an application of $k$-means++ to obtain the final results. 
For SCRLM+$k$-means, $k$-means is applied using the initial centers obtained from SCRLM. 

\subsection{Simulation Experiments}

This section shows experiments on synthetic data generated from a Gaussian mixture model with outliers described in Equation \ref{equ:GMM}.

\subsubsection{Comparison of Observed and Theoretical Accuracy}
\label{subsec:Theoretical bound comparison}

This section evaluates the tightness of the theoretical bounds for Algorithm \ref{alg:scRLM}. 
For simplicity, data is generated without outliers. 
The minimum and maximum weights for the positive clusters are taken to be $0.8/m$ and $1.2/m$ respectively. 
The standard deviations $\sigma_i$ of positive clusters are linearly increasing with $i$ from $1/16$ to $1/4$. 
The experiments use $\rho=0.5$.

The regions for different parameter combinations where the theoretical bound guarantees of achieving $100\%$ accuracy with at least $99\%$ probability are compared with similar regions obtained experimentally. 
The theoretical regions are described below on a case-by-case basis. 
The experimental regions are obtained by running  Algorithm \ref{alg:scRLM} with different parameter combinations on an exponential grid. 
For each parameter combination the algorithm is run 100 times and the number of times the algorithm has $100\%$ accuracy is recorded. 
The area where at least 99 of the 100 runs had $100\%$ accuracy is shown in light gray in Figure \ref{fig: Experimental Guarantee vs Theoretical Guarantee}.

Figure \ref{fig: Experimental Guarantee vs Theoretical Guarantee} a) displays the results for the data dimension $p$ vs. the sample size $N$, keeping the number of clusters $m$ fixed to $m=3$ and the subsample size $n=\lceil \frac{m}{a}(\log m+\log \frac{4}{\delta})\rceil$. 
According to Corollary \ref{cor7}, the theoretical $p$ in this case should be at least
\[
p > \lceil 128  (2\log {N} + \log \frac{40}{\delta}) \rceil
\] when $\delta=0.01$. 
This area is shown in dark gray in Figure \ref{fig: Experimental Guarantee vs Theoretical Guarantee} a).
From the plot, one could see that the theoretical bound on $p$ is not very tight, since there is a large gap, by a factor of over 64 between the dark region (theoretical) and the light gray region (experimental).
\begin{figure}[t]
\centering
\begin{tabular}{cc}
 \includegraphics[width=0.45\linewidth]{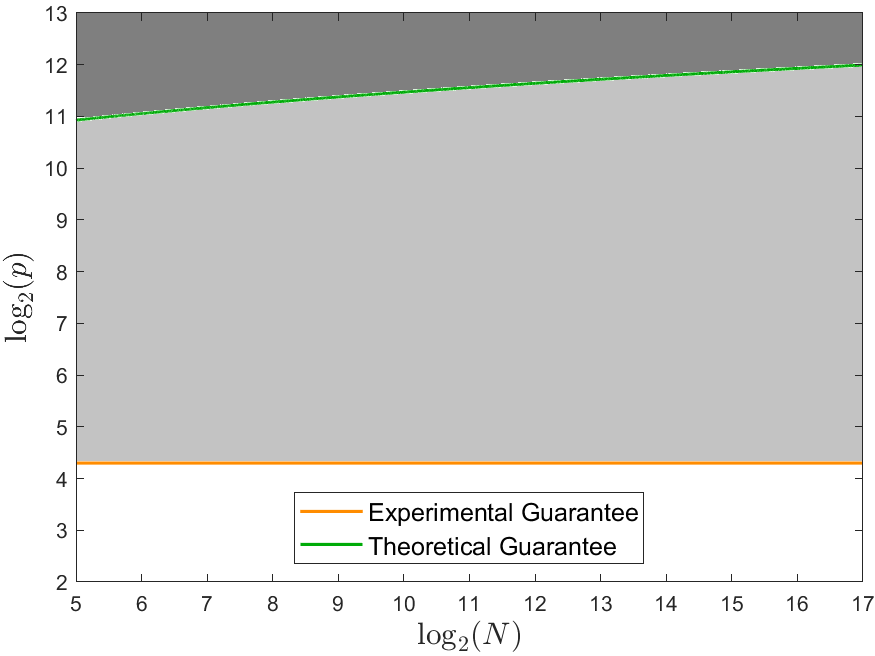}    
 &\includegraphics[width=0.45\linewidth]{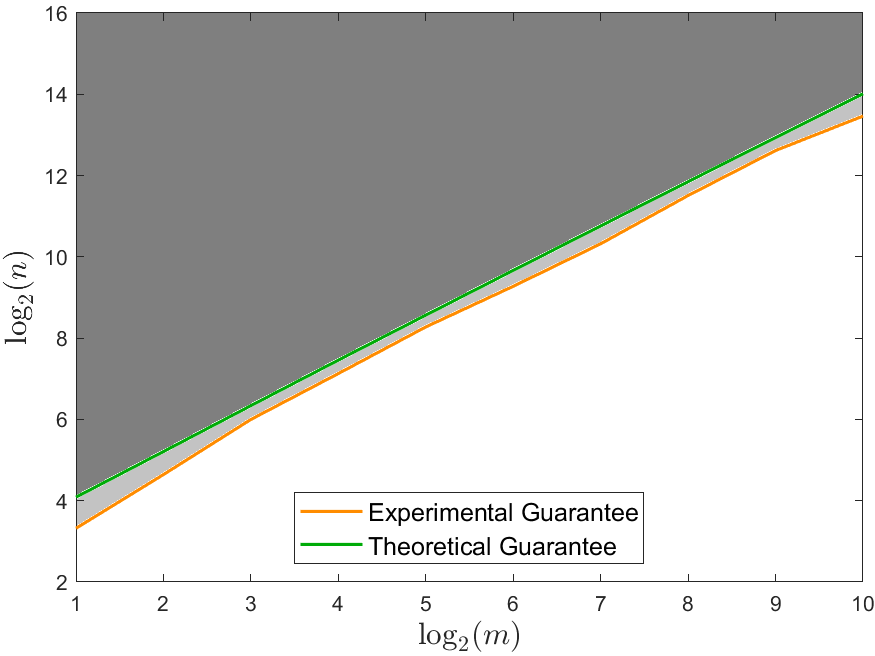}  \\
 a) $p$ vs $N  (m=3)$&b)  $n$ vs $m$  $(N=20,000, p=3,700)$\\
 \includegraphics[width=0.45\linewidth]{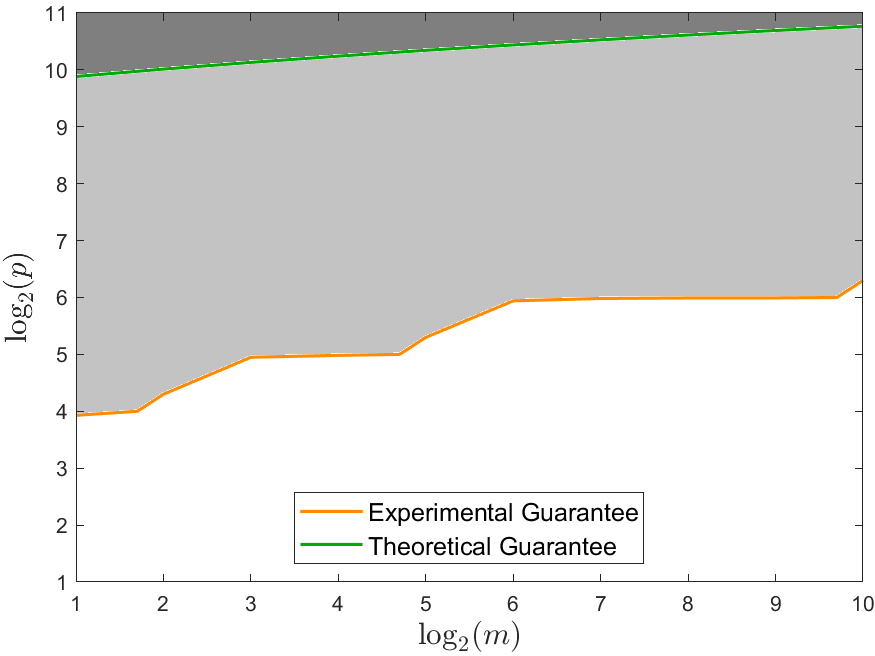}    
 & \includegraphics[width=0.45\linewidth]{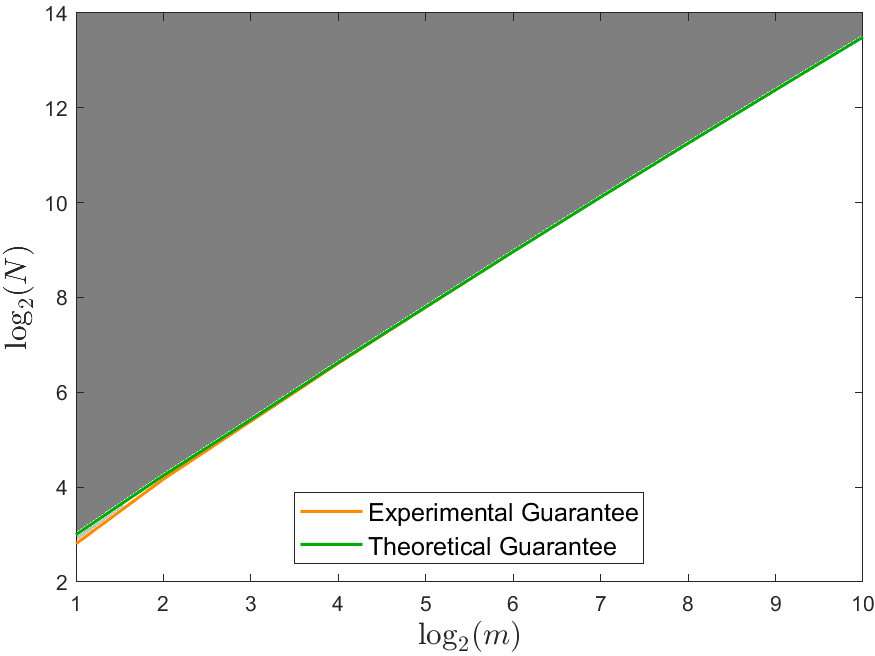}\\
c) $p$ vs $m$ $(N=20,000)$ &d) $N$ vs $m$ $(p=3,700)$
\end{tabular}
\vskip -3mm
\caption{Comparison between the parameter combinations where the SCRLM algorithm is theoretically guaranteed to have 100\% accuracy for 99\% of the time with the experimental findings.\label{fig: Experimental Guarantee vs Theoretical Guarantee}} 
\end{figure}

Figure \ref{fig: Experimental Guarantee vs Theoretical Guarantee} b) displays the results for the subsample size $n$ vs. the number of clusters $m$, when the sample size is $N=20,000$ and $p=3700$. 
According to Corollary \ref{cor7}, the theoretical $n$ is at least \[
n > \lceil \frac{m}{a}(\log m+\log \frac{4}{\delta})\rceil.
\]
The plot indicates that the theoretical bound for $n$ is tight, by a factor around 1.2.

Figure \ref{fig: Experimental Guarantee vs Theoretical Guarantee} c) displays the results for the data dimension $p$ vs. the number of clusters $m$, when $N$ is fixed to be $N=20,000$ and $n=\lceil \frac{m}{a}(\log m+\log \frac{4}{\delta})\rceil$.
According to Corollary \ref{cor7}, the theoretical $p$ should be at least
\[
p >\lceil128  (\log m +\log \frac{8}{\delta})\rceil
\] when $\delta=0.01$. 
The plot indicates that the bound on $p$ in not very tight, off by a factor over 32.

Figure \ref{fig: Experimental Guarantee vs Theoretical Guarantee} d) displays the results for the sample size $N$ vs. the number of clusters $m$, when $p=3700$ and $n=\lceil \frac{m}{a}(\log m+\log \frac{4}{\delta})\rceil
$.
According to Corollary \ref{cor7}, the theoretical $N$ is at least \[
N > \lceil \frac{m}{a}(\log m+\log \frac{4}{\delta})+1\rceil.
\]
Since the smallest $N$ one can pick is $n$, that explains why the theoretical bound almost overlaps the experimental bound in this case.

The empirical results support the conclusions that the theoretical bound for $p$ is conservative and accurate results are obtained with smaller values of $p$ in practice, but the theoretical bounds for $N$ and $n$ are in good agreement with values needed in practice.

\subsubsection{Stability of SCRLM w.r.t. the Bandwidth Parameter}
\label{subsec: Stability of SCRLM w.r.t. the bandwidth parameter rho}

\begin{figure}[t]
\centering
\begin{tabular}{cc}
 \includegraphics[width=0.45\linewidth]{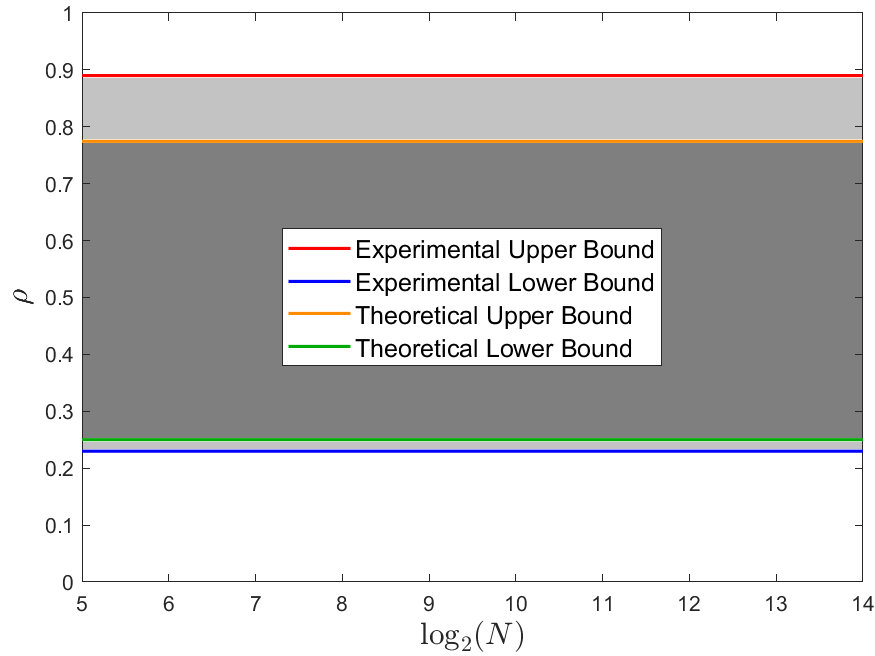}    
 &\includegraphics[width=0.45\linewidth]{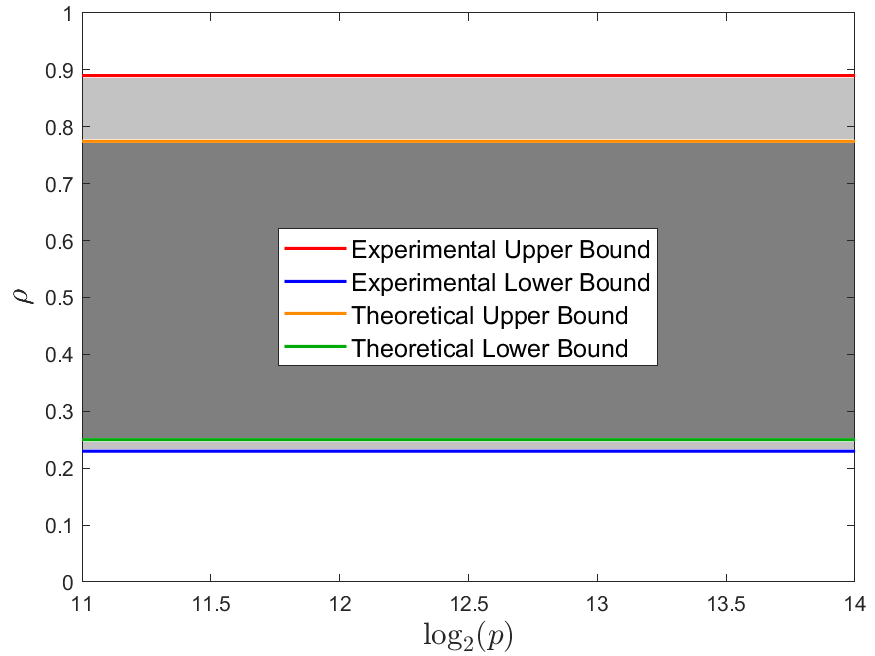}  \\
 a) $\rho$ vs $N$  $(m=3,p=3700)$ &b)  $\rho$ vs $p$  $(N=32, m=3)$\\
 \includegraphics[width=0.45\linewidth]{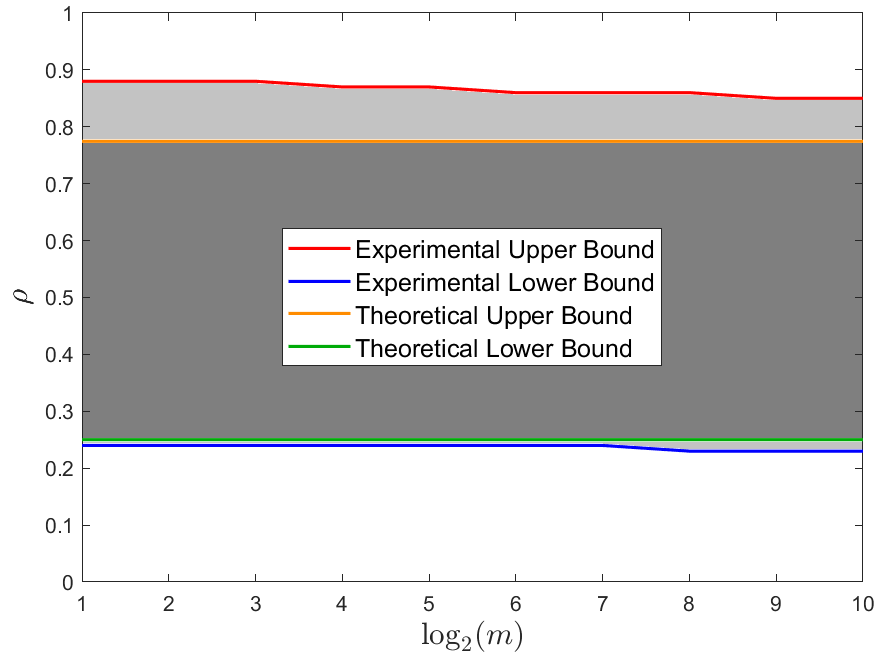}    
 & \includegraphics[width=0.45\linewidth]{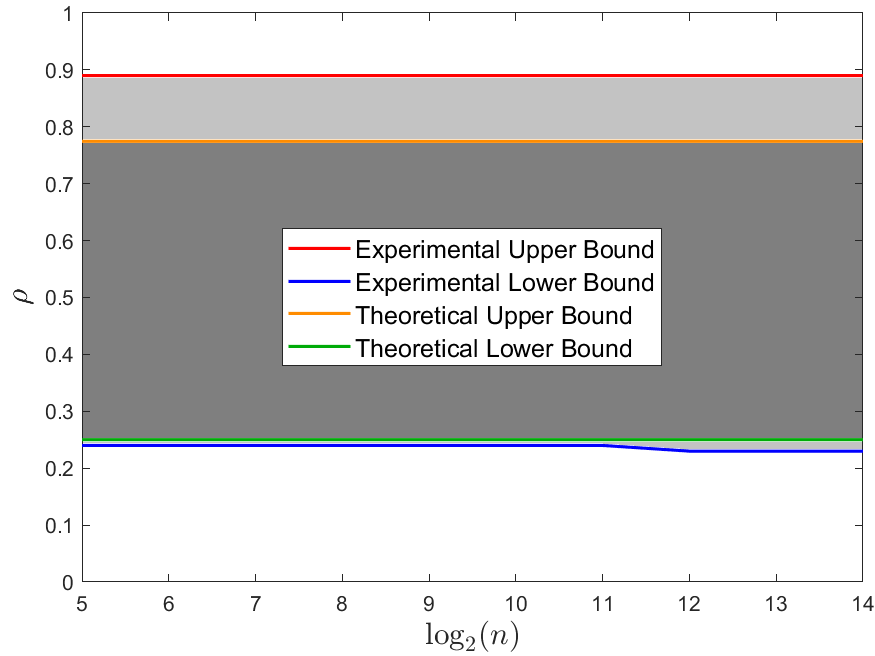}\\
c) $\rho$ vs $m$ $(N=20,000, p=3700)$  &d) $\rho$ vs $n$ $(N=20,000, p =3700, m=3)$
\end{tabular}
\vskip -3mm
\caption{Evaluation of tightness of the bandwidth parameter $\rho$.\label{fig: Evaluation of tightness of rho}} 
\end{figure}

This following experiments evaluate the tightness of the theoretical bounds of $\rho$ for Algorithm \ref{alg:scRLM}. The experiments use $\sigma_{max}=0.25$.

Figure \ref{fig: Evaluation of tightness of rho} a) displays the results for the bandwidth parameter $\rho$ vs. the sample size $N$, keeping the number of clusters $m$ fixed to $m=3$ and the subsample size $n=\lceil \frac{m}{a}(\log m+\log \frac{4}{\delta})\rceil$. 

Figure \ref{fig: Evaluation of tightness of rho} b) displays the results for the bandwidth parameter $\rho$ vs. the data dimension $p$, when $N=32$, $m=3$ and $n=\lceil \frac{m}{a}(\log m+\log \frac{4}{\delta})\rceil$. 

Figure \ref{fig: Evaluation of tightness of rho} c) displays the results for the the bandwidth parameter $\rho$ vs. the number of clusters $m$, when $N$ is fixed to be $N=20000$, $p$ is fixed to be $3700$ and $n=\lceil \frac{m}{a}(\log m+\log \frac{4}{\delta})\rceil$.

Figure \ref{fig: Evaluation of tightness of rho} d) displays the results for the bandwidth parameter $\rho$ vs. the number of subsamples $n$, when $N=20000$, $p=4200$ and $m=3$.

According to Assumption \ref{assumption1}, for all of the experiments, the theoretical upper bound of $\rho$ is $\sqrt{0.6}$, and the theoretical lower bound of $\rho$ is $\sigma_{max}=0.25$. 
From Figure \ref{fig: Evaluation of tightness of rho}, one could see that the theoretical upper bound on $\rho$ is not very tight with a difference of more than 0.1, but the theoretical lower bound on $\rho$ is very tight with the difference less than 0.02.

The empirical results support the conclusions that the theoretical upper bound for $\rho$ is not tight, that 100\% accuracy can be achieved with $\rho>\sigma_{max}$ in practice, but the theoretical lower bounds for $\rho$ are in good agreement with values needed in practice.

\begin{figure}[t]
\centering
\begin{tabular}{ccc}
 \includegraphics[width=0.3\linewidth]{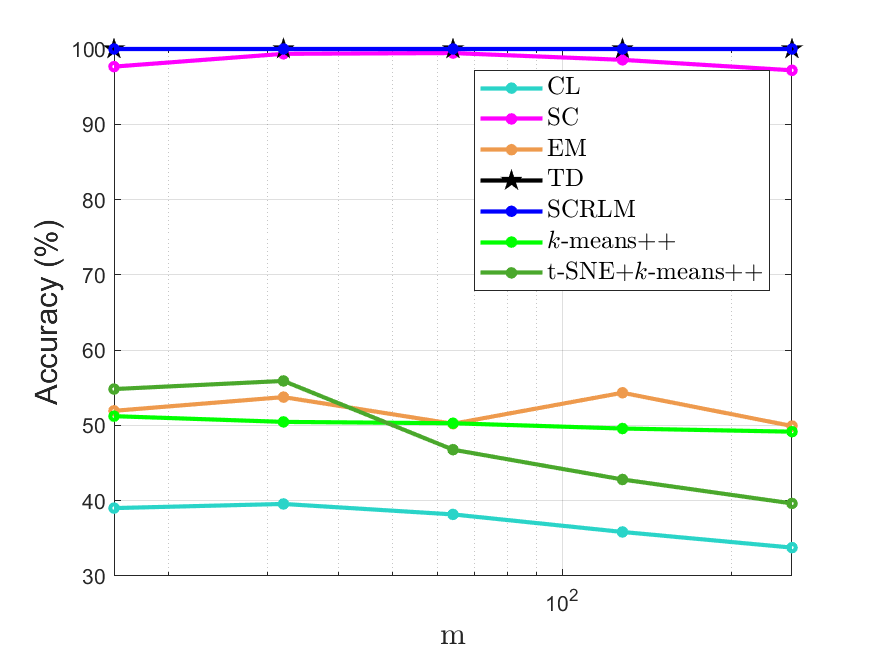}    
 &\includegraphics[width=0.3\linewidth]{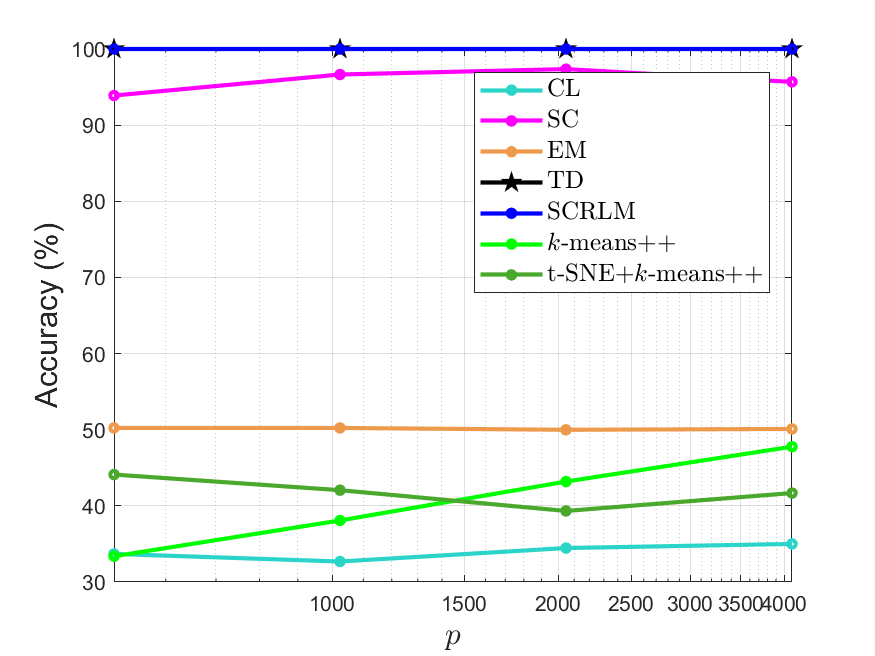}
  &\includegraphics[width=0.3\linewidth]{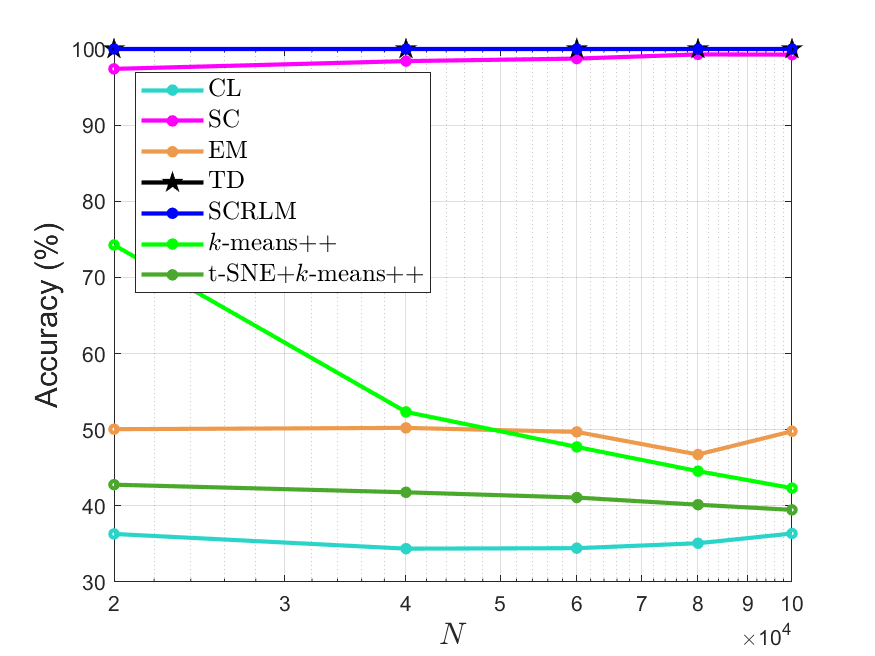}\\
 a) Accuracy vs $m$ &b)  Accuracy vs $p$ & c) Accuracy vs $N$\\
\end{tabular}
\vskip -3mm
\caption{Accuracy of clustering algorithms on simulation data. \label{fig:Acc_in_GMM}}
\end{figure}

\subsubsection{Comparison with other clustering methods}

For these simulations, the data is generated with different number of clusters ($m$), different dimension ($p$) and different number of observations ($N$). 
The data is generated to contain $50\%$ positives and $50\%$ negatives (outliers). The number of desired clusters is  specified as $m+1$ for the other methods evaluated besides SCRLM. 
For SCRLM, the number of desired clusters $T$ was selected to be $T=N$ and thus the actual number of clusters was found automatically.
From Figure \ref{fig:Acc_in_GMM}, one could see that only SCRLM, SC and TD are able to detect outliers, the other methods are very sensitive to outliers. 
SCRLM and TD achieve $100\%$ accuracy in all cases.

\subsection{Real Data Experiments}

To show that the SCRLM is an effective method, it was applied to four real datasets: the  MNIST \citep{deng2012mnist}, CIFAR-10 \citep{krizhevsky2009learning}, CIFAR-100 \citep{krizhevsky2009learning} and ImageNet ILSVRC-2012 dataset \citep{russakovsky2015imagenet}.

MNIST \citep{deng2012mnist} has 70,000 images of handwritten digits from 0 to 9 with 60,000 images used for training and 10,000 images used for testing. 
CIFAR-10 \citep{krizhevsky2009learning} consists of 60000 images in 10 classes, with 6000 images per class. 
There are 50000 training images and 10000 test images. 
CIFAR-100 \citep{krizhevsky2009learning} is just like the CIFAR-10, except it has 100 classes containing 600 images each. 
The ImageNet \citep{russakovsky2015imagenet} validation dataset has 50000 observations on 1000 classes with 50 observation per class and the ImageNet training dataset has almost 1.3 million observations on 1000 classes. 

\begin{figure}[t]
	\centering	\includegraphics[width=0.8\linewidth]{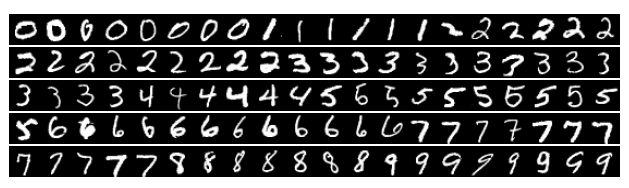} 
\vskip -3mm
\caption{Variability of MNIST, cluster centers obtained by SCRLM ($T = 100$). \label{fig:Variability_MNIST} }
\end{figure} 

\textbf{Data preprocessing.} Feature extraction for image data obtains a compact feature vector from the interesting parts of an image. 
The model SimCLR \citep{chen2020simple} was used to obtain a version of the MNIST dataset as real vectors with dimension $p=512$. 
The images from the CIFAR-10 and CIFAR-100 were resized to $144 \times 144$ pixels, then a pre-trained CNN, CLIP ResNet$50 \times 64$  \citep{radford2021learning} with average pooling was used to obtain a $p=4096$ dimensional feature vector for each image. 
The images from the ImageNet were resized to $224 \times 224$ pixels, then a $p=640$
dimensional feature vector for each image was obtained using CLIP ResNet$50 \times 4$ \citep{radford2021learning} and attention pooling.

\textbf{Results.} Figure \ref{fig:Variability_MNIST} shows the cluster centers obtained by SCRLM when the number of desired clusters $T$ is set to be 100 in MNIST. 
One could see that each cluster center is a good representation of that cluster. 
The variations of simple digits like 1 and 4 are relatively small, while complex digits like 2 and 3 have more variations. 
This shows MNIST is likely to have a hierarchical structure that can be used to cluster data when the number of clusters has a range of values.
\begin{figure}[t]
\centering
\begin{tabular}{cccc}
 \includegraphics[width=0.45\linewidth]{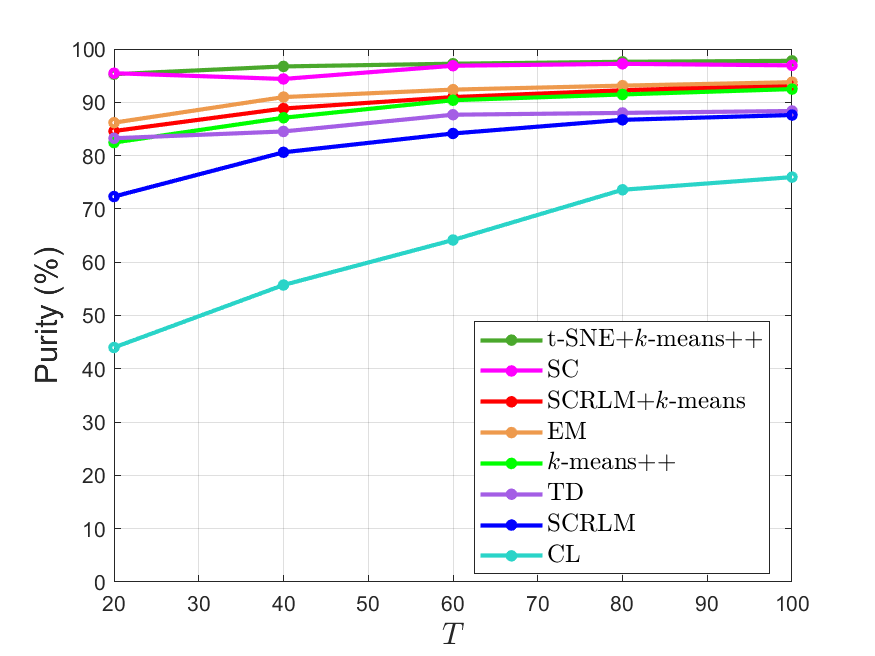}    
 &\includegraphics[width=0.45\linewidth]{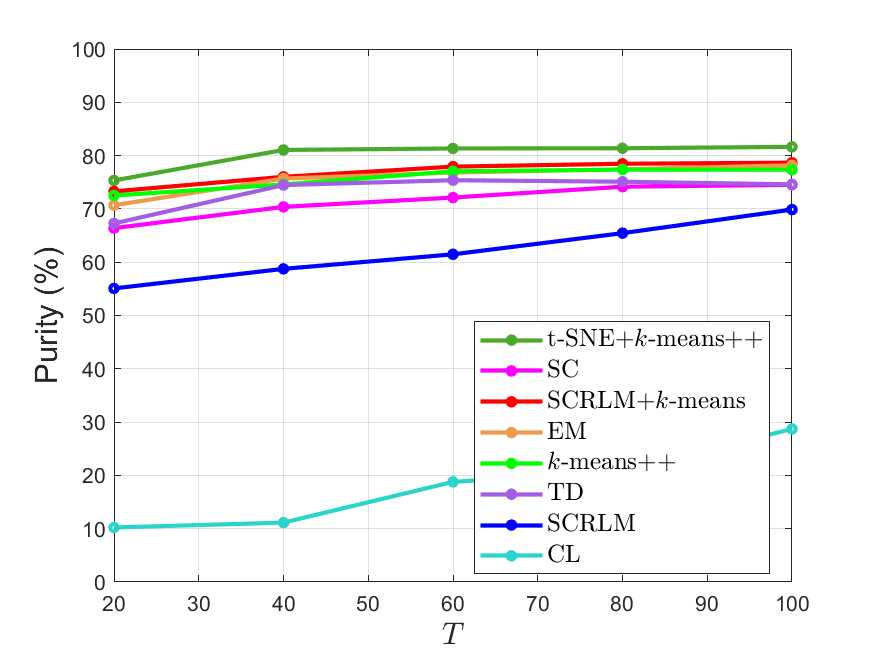}\vspace{-3mm}\\
 a) MNIST & b)  CIFAR-10 \\
\end{tabular}
\vskip -4mm
\caption{Purity vs number of clusters $T$ of clustering algorithms on MNIST and CIFAR-10. \label{fig:Acc_vs_T}}
\vspace{-4mm}
\end{figure}

\begin{figure}[t]
\vspace{-4mm}
\centering
\begin{tabular}{cccc}
 \includegraphics[width=0.45\linewidth]{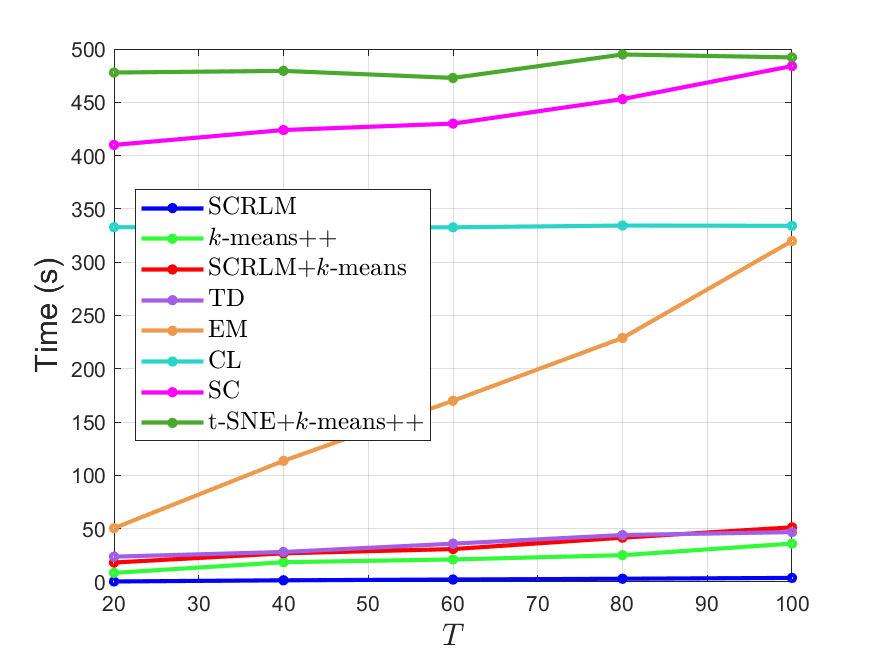}    
 &\includegraphics[width=0.45\linewidth]{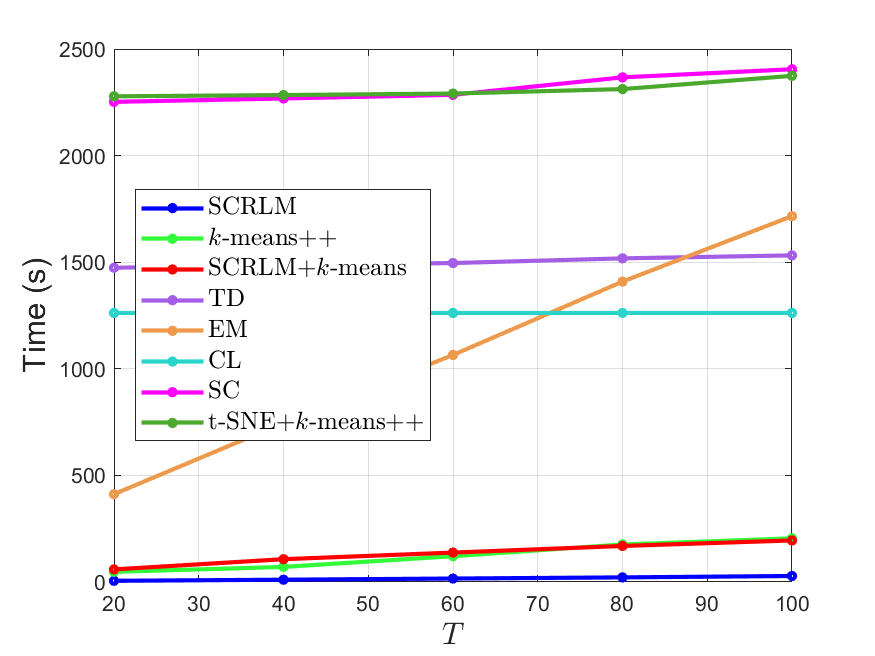}\vspace{-3mm}\\
 a) MNIST & b)  CIFAR-10 \\
\end{tabular}
\vskip -4mm
\caption{Computation time vs number of clusters $T$ of clustering algorithms on MNIST and CIFAR-10. \label{fig:Time_vs_T}}
\vspace{-5mm}
\end{figure}

Figure \ref{fig:Acc_vs_T} and \ref{fig:Time_vs_T} support the conclusion that the SCLRM-based methods are superior to other methods for problems with a large number of clusters. 
From the plot, one could see that the purity of SCRLM and SCRLM+$k$-means increases as the number of clusters increases. 
However, the purity of TD does not have an obvious increase as the number of clusters increases, and the running time of EM increases significantly as the number of clusters increases. 
Therefore, SCRLM+$k$-means is the most efficient in producing a particular level of accuracy within a particular time.

\begin{figure}[t]
\centering
\begin{tabular}{cccc}
\hspace{-4mm} \includegraphics[width=0.25\linewidth]{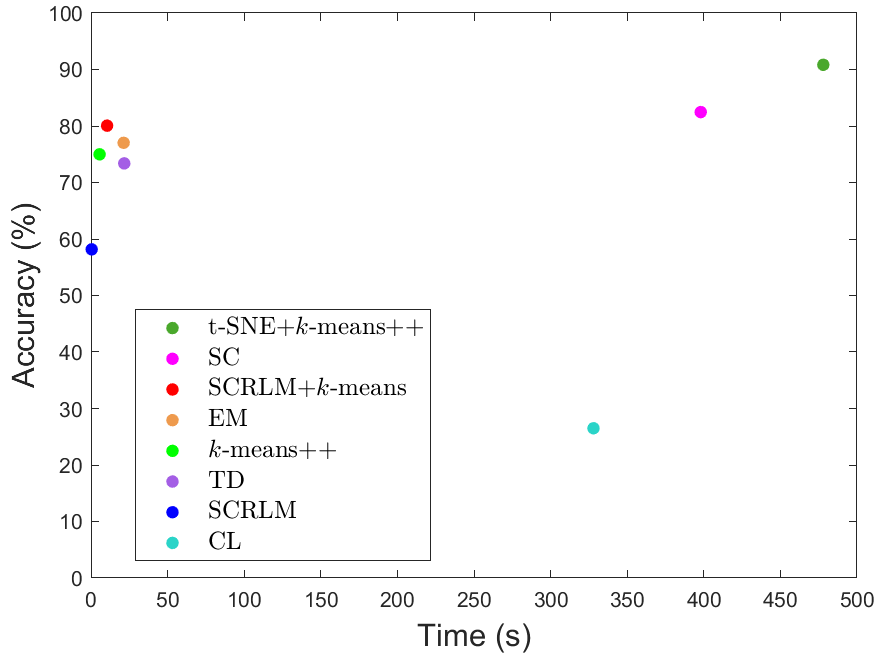}    
 &\hspace{-4mm}\includegraphics[width=0.25\linewidth]{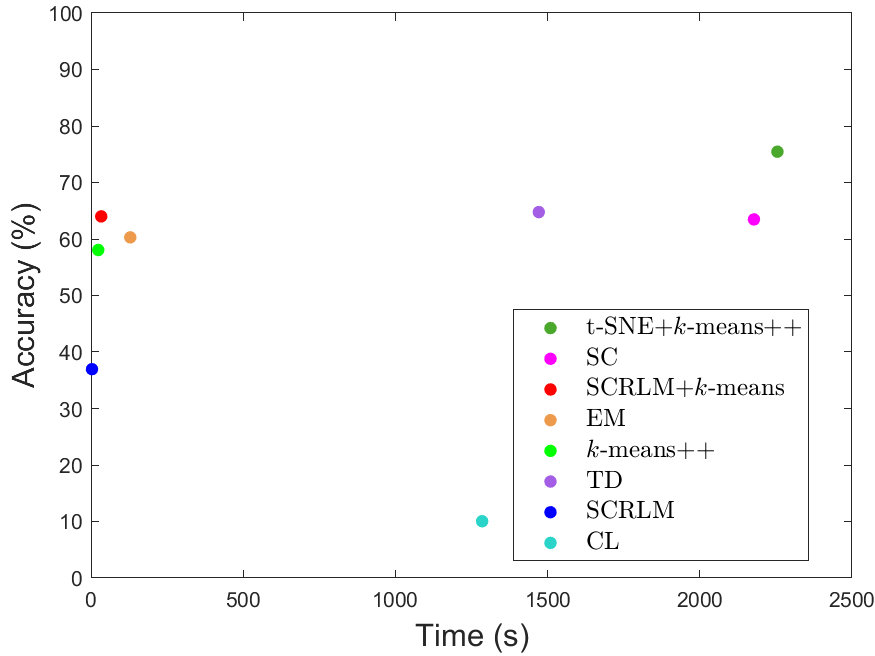} & 
 \hspace{-4mm}\includegraphics[width=0.25\linewidth]{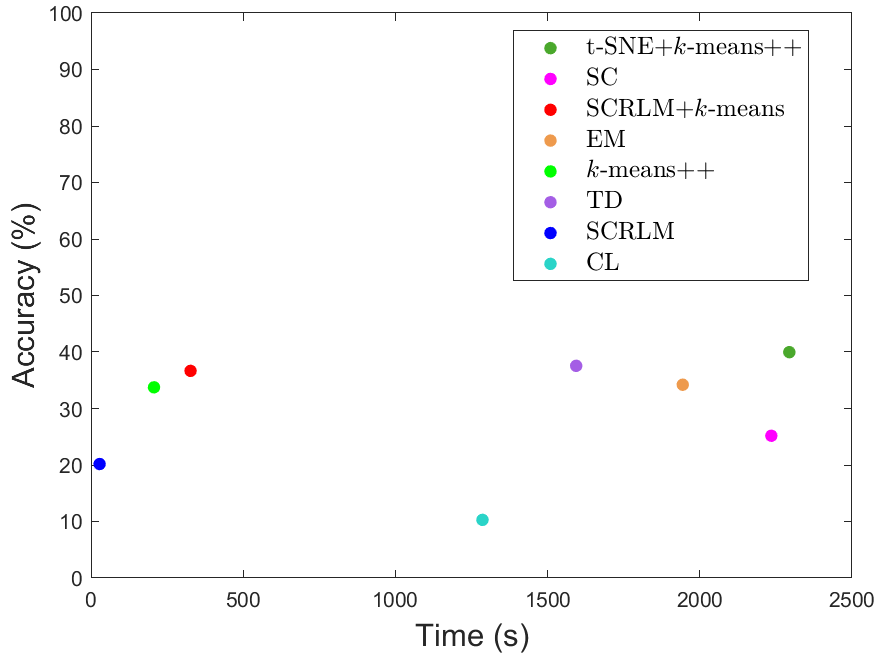} &
  \hspace{-4mm}\includegraphics[width=0.25\linewidth]{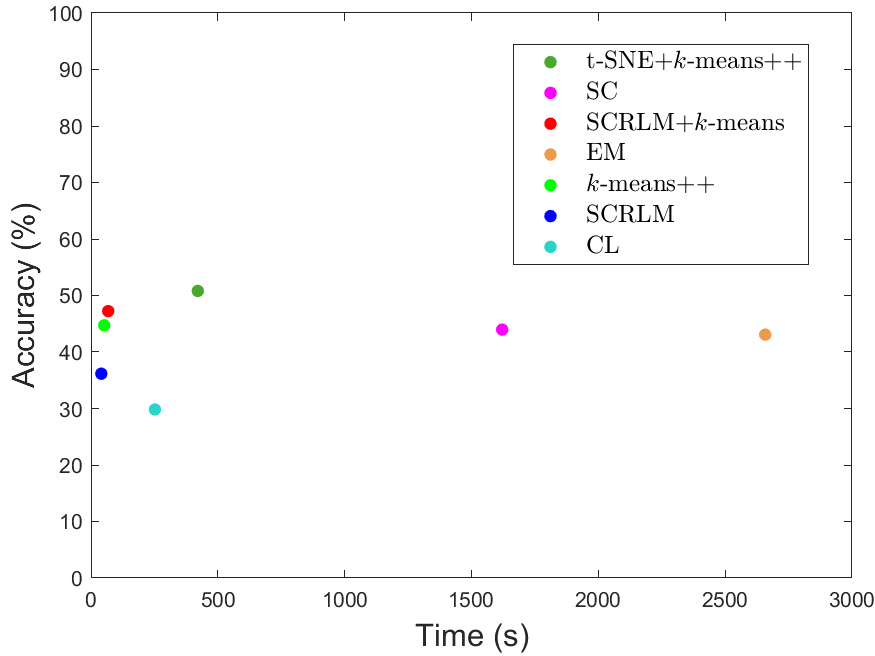} \\
 a)  MNIST & b)  CIFAR-10 & c)  CIFAR-100 & d) ImageNet Val\\
\end{tabular}
\vskip -3mm
\caption{Accuracy vs time of different clustering algorithms on four image datasets. \label{fig:Acc_vs_Time}}
\end{figure}
The comparison of accuracy and time is shown in Figure \ref{fig:Acc_vs_Time} and summarized in Tables \ref{table:comparison of accuracy} and \ref{table:comparison of time}. 
In all the cases, SCRLM outperforms all other methods in terms of running time. EM performs well when the number of clusters is small but has prohibitive computation cost for CIFAR-100 and ImageNet validation datasets. 
t-SNE and TD achieve the best accuracy but only have acceptable running time when the dimension is small. 
Therefore, only SCRLM, SCLRM+$k$-means and $k$-means++ are compared for the ImageNet training dataset. 
From Tables \ref{table:comparison of accuracy} and \ref{table:comparison of time} one could see that SCLRM+$k$-means achieves a higher accuracy on ImageNet than $k$-means++ in far less time, by a factor of 3.83. 
This demonstrates that SCRLM can be used as an initialization technique for $k$-means clustering that has a better performance than $k$-means++.

\begin{table}[!ht]
\centering
\begin{tabular}{|c|c|c|c|c|c|}
\hline
Accuracy(\%)    & MNIST & CIFAR-10&CIFAR-100&ImageNet val & ImageNet \\ \hline
CL  & 26.50  & 10.05  & 10.29   & 29.83   & - \\ \hline
SC  & 82.46  & 63.47  & 25.17   & 43.96 & -       \\ \hline
EM  & 77.03  & 60.29  & 34.21  & 43.07 & -    \\ \hline
TD  & 73.38  & 64.76  & 37.55  & - & -     \\ \hline
t-SNE+$k$-means++  & 90.83 & 75.45  & 39.97 & 50.81 & -  \\ \hline
$k$-means++  & 74.99   & 58.06 & 33.75  &  44.73 &  47.71    \\ \hline
SCRLM  & 58.17 &  36.96 &  20.17   & 36.16  &    34.01   \\ \hline
SCRLM+$k$-means & 80.06 &  64.00 &  36.66 & 47.24 &  48.61    \\ \hline
\end{tabular}
\caption{Accuracy of clustering algorithms on five image datasets. \label{table:comparison of accuracy}}
\end{table}

\begin{table}[!ht]
\centering
\begin{tabular}{|c|c|c|c|c|c|}
\hline
Time(s)    & MNIST  & CIFAR-10& CIFAR-100& ImageNet val & ImageNet \\ \hline
CL  & 328   &  1,285   & 1,286 &252  & -  \\ \hline
SC   & 398    & 2,178 &2,235   & 1,621    & -    \\ \hline
EM   & 21.3    &129 & 1,944&2,658    & -        \\ \hline
TD   & 21.7   &  1,471  &1,594 & -   & -    \\ \hline
t-SNE+$k$-means++  & 478  & 2,255  & 2,294  & 421   & -  \\ \hline
$k$-means++& 5.61  &23.8   & 207 &52.1 &  10,005    \\ \hline
SCRLM  & 0.46    &3.25  & 28.2 &40.6   &1,269       \\ \hline
SCRLM+$k$-means & 10.5   & 33.6  &327 & 67.8 & 2,610\\ \hline
\end{tabular}
\caption{Computation time of clustering algorithms on five image datasets. \label{table:comparison of time}}
\end{table}

\section{Conclusion}
\label{sec:conc}
In this paper, a novel algorithm named SCRLM is proposed for clustering large scale Gaussian mixture models with outliers. 
The basic assumptions of the algorithm are: isotropic Gaussians for the foreground (positives) clusters, and a constraint on the range of values of the bandwidth parameter $\rho$ of the loss function. 
Unlike most clustering methods, the algorithm has strong theoretical guarantees that, with high probability, it is able to detect all the outliers and cluster all the observations correctly. 
Theoretical and numerical results confirm that SCRLM is an effective clustering method when the number of clusters and dimension are large. 
Moreover, it can be used as an initialization strategy for $k$-means clustering and was observed to have better performance than other centroid initialization methods in extensive experiments. 

There are still some drawbacks of SCRLM that must be overcome with additional work in the future. 
First, the clustering results of SCRLM depend strongly on the bandwidth parameter $\rho$ in the loss function. 
Its value is currently determined by trial and error. 
Second, it was observed that the clustering results of SCRLM for large numbers of clusters are more satisfactory than for small numbers of clusters. 
Hence, the future work will focus on two aspects. 
First, strategies for determining an effective value of $\rho$ based on the distribution assumptions and the given data will be explored. 
Second, a hierarchical clustering method based on SCLRM that is able to handle a large numbers of clusters, on the order of tens of thousands to millions will be designed and evaluated.
\bigskip
\begin{center}
{\large\bf SUPPLEMENTARY MATERIALS}
\end{center}
In the supplement, the basic separation and concentration results for pairs of training examples are presented in Appendix \ref{sec:A}. Appendix \ref{sec:B} contains proofs of Proposition \ref{prop1} and Lemma \ref{lemma2}, and Appendix \ref{sec:C} contains proofs of the basic propositions on loss bounds. The proofs of Theorem \ref{thm1} and Corollary \ref{cor7} are given in Appendix \ref{sec:D}.
\bibliographystyle{jabes}
\bibliography{myrefs.bib}

\appendix
\section{Preliminaries}
\label{sec:A}
\begin{lemma}(From \citep{wainwright2019high}, Example 2.5)\label{lem:std}
If  $Z_1,...,Z_n$ are i.i.d Gaussian random variables $Z_i \sim \mathcal{N}(0,1)$,  then for any $\epsilon \in (0,1)$,
$$
\mathbb{P}\left(\left|\frac{1}{n}\sum_{i=1}^n Z_i^2-1 \right| \geq \epsilon\right) \leq 2\exp \left\{ - n \epsilon^2/8 \right\}.
$$
\end{lemma}

\begin{corollary}\label{cor1}
If $\bx=(X_1,...,X_p)$ is a multivariate Gaussian random variable $\bx\sim \mathcal{N}(\mathbf{0},I_p)$, then $\mathbb{E}\left(\|\bx\|^2\right)=p$ and for any $\epsilon \in (0,1)$,
$$
\mathbb{P}\left(\left|\frac{1}{p}\|\bx\|^2-1 \right| \geq \epsilon\right) \leq 2\exp\{ - p \epsilon^2/8\}.
$$
\end{corollary}
\begin{proof}
Follows from Lemma \ref{lem:std} above taking $Z_i=X_i, i=1,...,p$.
\end{proof}
\begin{corollary}\label{cor2}
If $\bx=(X_1,...,X_p), \by=(Y_1,...,Y_p)$ are independent multivariate Gaussian random variables $\bx,\by\sim \mathcal{N}(\mathbf{0},I_p)$, then $\mathbb{E}\left(\|\bx-\by\|^2\right)=2p$ and for any $\epsilon \in (0,1)$,
$$
\mathbb{P}\left(\left|\frac{1}{2p}\|\bx-\by\|^2-1 \right| \geq \epsilon\right) \leq 2\exp\{ - p \epsilon^2/8\}.
$$
\end{corollary}
\begin{proof}
Follows from Lemma \ref{lem:std} above taking $Z_i=(X_i-Y_i)/\sqrt{2}, i=1,...,p$.
\end{proof}

Using these results, it follows that with high probability the negatives are far away from each other.
\begin{corollary}[Separation between negatives] \label{cor:sep_neg}
 For two negatives $\bx_i$ and $\bx_k$, with probability at least $1-2\exp\{ -p/128\}$, the separation satisfies
$$
\left\|\bx_{i}-\bx_{k}\right\|^2>1.5p,  \quad \bx_{i},\bx_{k}\in H.
$$
\end{corollary}
\begin{proof}
Since $\bx_k \sim \mathcal{N}(\mathbf{0},I_p)$ and $\bx_i \sim \mathcal{N}(\mathbf{0},I_p)$, then  $\bx_i-\bx_k\sim \mathcal{N}(\mathbf{0},2I_p)$, thus $\mathbb{E}\left(\|\bx_i-\bx_k\|^2\right)=2p$.
According to Corollary \ref{cor2}, it follows that
 $$
\mathbb{P}\left(\left|\frac{\|\bx_i-\bx_k\|^2}{2p}-1 \right| \geq \epsilon\right) \leq 2\exp\left\{ - p\epsilon^2 /8\right\},
$$
then
 $$
\mathbb{P}\left(\|\bx_i-\bx_k\|^2\leq 2p(1 - \epsilon)\right) \leq 2\exp\left\{ - p\epsilon^2 /8\right\}.
$$
Then with high probability at least $1-2\exp\{ - p\epsilon^2/8\}$, the separation satisfies
$$
\|\bx_i-\bx_k\|^2> 2p(1 - \epsilon).
$$
Now take $\epsilon=1/4$ so that with high probability at least $1-2\exp\{ - p/128\}$, the separation satisfies
$$
\|\bx_i-\bx_k\|^2>1.5p.
$$
\end{proof} 
It then follows that the positives from the same cluster are within a certain radius from each other with high probability.
\begin{corollary}[Concentration of positives in the same cluster]\label{cor:con_pos}
For any positive cluster $S_{j}$ with mean $\bmu_{j}$ and covariance matrix $\sigma_{j}^{2} I_{p}$, with probability at least $1-2\exp\{-p/128\}$, the concentration is bounded as
$$
\left\|\bx_{i}-\bx_{k}\right\|^2<2.5 p\sigma_{j}^2, \quad \bx_{i},\bx_{k}\in S_{j}.
$$
\end{corollary}
\begin{proof}
Since $\bx_i \sim \mathcal{N}(\bmu_j,\sigma_j^2I_p)$ and $\bx_k \sim \mathcal{N}(\bmu_j,\sigma_j^2I_p)$, then  $\bx_i-\bx_k\sim \mathcal{N}(\mathbf{0},2\sigma_j^2I_p)$, thus $
\mathbb{E}\left(\|\bx_i-\bx_k\|^{2}\right)=2p\sigma_{j}^{2}
$. According to Corollary \ref{cor1}, it follows that
$$
\mathbb{P}\left(\left|\frac{\|\bx_i-\bx_k\|^2}{2p\sigma_j^2}-1 \right| \geq \epsilon\right) \leq 2\exp\left\{ - p\epsilon^2 /8\right\},
$$
then
$$
\mathbb{P}(\|\bx_i-\bx_k\|^2\geq 2p\sigma_j^2(1+\epsilon)) \leq 2\exp\{ - p \epsilon^2/8\}.
$$
Take $\epsilon=1/4$, yields
$$
\mathbb{P}(\|\bx_i-\bx_k\|^2\geq 2.5p\sigma_j^2) \leq 2\exp\{ - p/128\}.
$$
Therefore, with probability at least 
$1-2\exp\{ - p/128\}$, the concentration is bounded as
\[\|\bx_i-\bx_k\|^2 < 2.5 p \sigma_j^2.\]
\end{proof}

We then prove that the positives are far away from the negatives with high probability.
\begin{corollary}[Separation between positives and negatives] \label{cor:sep_pos_neg}
For negative $\bx_i$ and positive $\bx_k$ from cluster $S_{j}$ with mean $\bmu_{j}$ and covariance matrix $\sigma_{j}^{2} I_{p}$, with probability at least $1-2\exp\{ -p/128\}$, the separation satisfies
$$
\left\|\bx_{i}-\bx_{k}\right\|^2>p(1.5+0.75\sigma_{j}^2), \quad \bx_{i}\in H, \bx_{k}\in S_{j}.
$$
\end{corollary}

\begin{proof}
Since $\bx_k \sim \mathcal{N}(\bmu_j,\sigma_j^2I_p)$ and $\bx_i \sim \mathcal{N}(\mathbf{0},I_p)$, then  $\bx_i-\bx_k\sim \mathcal{N}(\bmu_j,\sigma_j^2I_p+I_p)$, thus $\bx_i-\bx_k=\bmu_j+\boldsymbol{\epsilon_1}\sqrt{\sigma_j^2+1}$ with $\boldsymbol{\epsilon_1} \sim \mathcal{N}\left(\mathbf{0}, I_{p}\right).$ Since $\bmu_j\sim \mathcal{N}\left(\mathbf{0}, I_{p}\right)$, then $\bx_i-\bx_k$ is a Gaussian with 
$\mathbb{E}\left(\bx_i-\bx_k\right)=\mathbf{0}$ and 
$$
\mathbb{E}\left(\|\bx_i-\bx_k\|^{2}\right)=\mathbb{E}\left[\left(\bmu_j+\boldsymbol{\epsilon_1}\sqrt{\sigma_j^2+1}\right)^{T}\left(\bmu_j+\boldsymbol{\epsilon_1}\sqrt{\sigma_j^2+1}\right)\right]=
\mathbb{E}(\|\bmu_j\|^{2})+(\sigma_{j}^{2}+1) \mathbb{E}\left(\boldsymbol{\epsilon_1}^{T} \boldsymbol{\epsilon_1}\right),
$$
thus
$$
\mathbb{E}\left(\|\bx_i-\bx_k\|^{2}\right)=p+(\sigma_{j}^{2}+1) \mathbb{E}\left(\|\boldsymbol{\epsilon_1}\|^{2}\right)=
(2+\sigma_{j}^{2}) p.
$$
According to Corollary \ref{cor1}, it follows immediately that,
 $$
\mathbb{P}\left(\left|\frac{\|\bx_i-\bx_k\|^2}{(2+\sigma_j^2)p}-1 \right| \geq \epsilon\right) \leq 2\exp\left\{ - p\epsilon^2 /8\right\},
$$
then
$$
\mathbb{P}\left(\|\bx_i-\bx_k\|^2\leq p(2+\sigma_j^2)(1 - \epsilon)\right) \leq 2\exp\left\{ - p\epsilon^2 /8\right\}.
$$
Then with probability at least $1-2\exp\{ - p\epsilon^2/8\}$, the separation satisfies
$$
\|\bx_i-\bx_k\|^2> p(2+\sigma_j^2)(1 - \epsilon).
$$
Now take $\epsilon=1/4$, so that with probability at least $1-2\exp\{ - p/128\}$, the separation satisfies 
$$
\|\bx_i-\bx_k\|^2> p(1.5+0.75\sigma_{j}^{2}).
$$
\end{proof}

Moreover, positives from different clusters are also far from each other with high probability.
\begin{corollary}[Separation between positives in different clusters]\label{cor:sep_pos}
For positive $\bx_i$ from cluster $S_{i}$ with true mean $\boldsymbol{\mu}_{i}$ and covariance matrix $\sigma_{i}^{2} I_{p}$ and positive $\bx_k$ from another cluster $S_{j}$ with true mean $\boldsymbol{\mu}_{j}$ and covariance matrix $\sigma_{j}^{2} I_{p}$, with probability at least $1-2\exp\{ - p/128\}$, the separation satisfies
$$
\|{\bx_i}-{\bx_k}\|^2> p(1.5+0.75 \sigma_{i}^{2}+0.75\sigma_{j}^{2}),
\quad \mathbf{x}_{i} \in S_{i}, \mathbf{x}_{k} \in S_{j}.
$$
\end{corollary}
\begin{proof}
Since $\bx_k \sim \mathcal{N}(\bmu_j,\sigma_j^2I_p)$ and $\bx_i \sim \mathcal{N}(\bmu_i,\sigma_i^2I_p)$, then  $\bx_i-\bx_k\sim \mathcal{N}(\bmu_j-\bmu_i,\sigma_j^2I_p+\sigma_i^2I_p)$, thus $\bx_i-\bx_k=\bmu_j-\bmu_i+\boldsymbol{\epsilon_1}\sqrt{\sigma_j^2+\sigma_i^2}$ with $\boldsymbol{\epsilon_1} \sim \mathcal{N}\left(\mathbf{0}, I_{p}\right).$ Since $\bmu_j\sim \mathcal{N}\left(\mathbf{0}, I_{p}\right)$ and $\bmu_i\sim \mathcal{N}\left(\mathbf{0}, I_{p}\right)$, then $\bmu_j - \bmu_i \sim \mathcal{N}\left(\mathbf{0}, 2I_{p}\right)$, then $\bx_i-\bx_k$ is a Gaussian with 
$\mathbb{E}\left(\bx_i-\bx_k\right)=\mathbf{0}$ and 
$$
\mathbb{E}\left(\|\bx_i-\bx_k\|^{2}\right)=\mathbb{E}\left[\left(\bmu_j-\bmu_i+\boldsymbol{\epsilon_1}\sqrt{\sigma_j^2+\sigma_i^2}\right)^{T}\left(\bmu_j-\bmu_i+\boldsymbol{\epsilon_1}\sqrt{\sigma_j^2+\sigma_i^2}\right)\right],
$$
thus
$$
\mathbb{E}\left(\|\bx_i-\bx_k\|^{2}\right)=\mathbb{E}(\|\bmu_j-\bmu_i\|^{2})+(\sigma_{j}^{2}+\sigma_i^2) \mathbb{E}\left(\boldsymbol{\epsilon_1}^{T} \boldsymbol{\epsilon_1}\right)=2p+(\sigma_{j}^{2}+\sigma_i^2) \mathbb{E}\left(\|\boldsymbol{\epsilon_1}\|^{2}\right)=
(2+\sigma_i^2+\sigma_{j}^{2}) p.
$$
According to Corollary \ref{cor1}, it follows that $$
\mathbb{P}\left(\left|\frac{\|\bx_i-\bx_k\|^2}{(2+\sigma_i^2+\sigma_j^2)p}-1 \right| \geq \epsilon\right) \leq 2\exp\left\{ - p\epsilon^2 /8\right\},
$$
then
$$
\mathbb{P}\left(\|\bx_i-\bx_k\|^2\leq p(2+\sigma_i^2+\sigma_j^2)(1 - \epsilon)\right) \leq 2\exp\left\{ - p\epsilon^2 /8\right\}.
$$
Then with probability at least $1-2\exp\{ - p\epsilon^2/8\}$, the separation satisfies
$$
\|\bx_i-\bx_k\|^2> p(2+\sigma_i^2+\sigma_j^2)(1 - \epsilon).
$$
Now take $\epsilon=1/4$ so that with probability at least $1-2\exp\{ - p/128\}$, the separation satisfies 
$$
\|\bx_i-\bx_k\|^2> p(1.5+0.75\sigma_i^2+0.75\sigma_{j}^{2}).
$$
\end{proof}
The previous corollaries are used to prove that with high probability, all positives from each cluster are within $2.5 p\rho^2$ of each other, and $2.5 p\rho^2$ away from the other clusters and from the negatives.

\section{Proof of Key Lemmas and Propositions }
\label{sec:B}
\begin{proposition}\label{prop1}
Given $N$ samples from a GMM with outliers,  and $\sigma_{max}\leq \rho < \sqrt{0.6}$, then with probability at least $1-6N^2\exp\{-p/128\}$, the distance between positives within a cluster satisfies
$$\left\|\bx_{i}-\bx_{j}\right\|^2<2.5p\rho^2, \quad \forall \bx_i, \bx_j  \text{ s.t. } l(\bx_i)=l(\bx_j)>0,$$ 
and the distance between positives from a cluster and other samples not in that cluster satisfies$$\left\|\bx_{i}-\bx_{j}\right\|^2>2.5p\rho^2 \quad \forall \bx_i, \bx_j
\text{ s.t. } l(\bx_j)\not =l(\bx_i)>0.$$ 
\end{proposition}
\begin{proof}
From Corollary \ref{cor:con_pos}, with probability at least $1-2\exp\{-p/128\}$, the distance between two positives in the same cluster is bounded as
$$
\left\|\bx_{i}-\bx_{j}\right\|^2<2.5 p\sigma_{l(\bx_i)}^2, \quad l(\bx_i)=l(\bx_j)>0.
$$
Using the union bound, with probability at least $1-2N^2\exp\{ -p/128\}$, the distances between all positives in the same cluster are bounded as
$$
\left\|\bx_{i}-\bx_{j}\right\|^2<2.5 p\sigma_{l(\bx_i)}^2\leq 2.5 p\rho^2,   \quad \forall \bx_i, \bx_j \text{ s.t. } l(\bx_i)=l(\bx_j)>0.
$$
From Corollary \ref{cor:sep_pos_neg}, with probability at least $1-2\exp\{ - p/128\}$, the distance between a positive and a negative satisfies
$$
\|{\bx_i}-{\bx_j}\|^2> p(1.5+0.75\sigma_{l(\bx_i)}^{2}),
\quad l(\bx_i)>0, l(\bx_j)=-1.
$$
Using the union bound, with probability at least $1-2N^2\exp\{-p/128\}$, the distance between any positive and negative satisfies
$$
\left\|\bx_{i}-\bx_{j}\right\|^2>p(1.5+0.75\sigma_{l(\bx_i)}^{2}), \quad \forall  \mathbf{x}_{i},  \bx_j, \text{ s.t. } l(\bx_i)>0, l(\bx_j)=-1.
$$
Given $\sigma_{max}\leq \rho < \sqrt{\frac{1.5+0.75\sigma_k^2}{2.5}}$, since $0<\sigma_k<1$, therefore, with $\sigma_{max}\leq\rho < \sqrt{\frac{1.5}{2.5}}= \sqrt{0.6}$, then with probability at least $1-2N^2\exp\{ -p/128\}$, the distance between any positive and negative  satisfies
$$
\left\|\bx_{i}-\bx_{j}\right\|^2>2.5p\rho^2, \quad \forall  \mathbf{x}_{i},  \bx_j, \text{ s.t. } l(\bx_i)>0, l(\bx_j)=-1.
$$
From Corollary \ref{cor:sep_pos}, with probability at least $1-2\exp\{ - p/128\}$, the distance between two positives from different clusters satisfies
$$
\|{\bx_i}-{\bx_j}\|^2> p(1.5+0.75 (\sigma_{l(\bx_i)}^{2}+\sigma_{l(\bx_j)}^{2})),
\quad  0<l(\bx_i)\not =l(\bx_j)>0.
$$
Using the union bound, with probability at least $1-2N^2\exp\{-p/128\}$, the distance between any two positives from different clusters satisfies
$$
\left\|\bx_{i}-\bx_{j}\right\|^2>p(1.5+0.75 (\sigma_{l(\bx_i)}^{2}+\sigma_{l(\bx_j)}^{2})), \quad \forall  \mathbf{x}_{i},  \mathbf{x}_{j}\text{ s.t. } 0<l(\bx_i)\not =l(\bx_j)>0.
$$
Given $\sigma_{max}\leq\rho < \sqrt{\frac{1.5+0.75(\sigma_i^2+\sigma_j^2)}{2.5}}$, since $0<\sigma_i,\sigma_j<1$, therefore, with $\sigma_{max}\leq\rho < \sqrt{0.6}$, then with probability at least $1-2N^2\exp\{ -p/128\}$, the distance between any two positives from different clusters satisfies
$$\left\|\bx_{i}-\bx_{j}\right\|^2>2.5p\rho^2, \quad \forall  \mathbf{x}_{i}, \mathbf{x}_{j} \text{ s.t. } 0<l(\bx_i)\not =l(\bx_j)>0.$$
Therefore, with probability at least $1-4N^2\exp\{ -p/128\}$, the distance between any positive and any sample not from that cluster satisfies
$$
\left\|\bx_{i}-\bx_{j}\right\|^2>2.5 p\rho^2, \quad \forall  \mathbf{x}_{i}, \mathbf{x}_{j} \text{ s.t. } l(\bx_j)\not =l(\bx_i)>0.
$$
Therefore, with probability at least $1-6N^2\exp\{ -p/128\}$, the following bounds on positives within a cluster and between clusters are satisfied
$$
\left\|\bx_{i}-\bx_{j}\right\|^2>2.5 p\rho^2, \quad \forall  \mathbf{x}_{i}, \mathbf{x}_{j} \text{ s.t. } l(\bx_j)\not =l(\bx_i)>0,
$$
and 
$$
\left\|\bx_{i}-\bx_{j}\right\|^2<2.5 p\rho^2, \quad \forall  \mathbf{x}_{i}, \mathbf{x}_{j} \text{ s.t. } l(\bx_j)=l(\bx_i)>0.
$$
\end{proof}

Finally, a bound for the probability that a sample $S$ has at least one element from each positive cluster is proven.
\begin{lemma}
\label{lemma2}
If the clusters have weights $w_1,...,w_m$, with $\sum_{k=1}^m w_k\leq 1$, then the probability that a sample $S$ of size $|S|=n$ contains at least one observation from each cluster is at least
\[
\mathbb{P}(|\{\bx\in S, l(\bx)=k\}|\geq 1, \forall k=\overline{1,m})\geq 1-\sum_{k=1}^m(1-w_k)^n.
\]
\end{lemma}
\begin{proof}
The probability that $S$ contains no elements from cluster $k$ is 
\[
\mathbb{P}(l(\bx)\not = k, \forall \bx\in S)=(1-w_k)^n.
\] 
Then using the union bound, the probability that there is a $k$ such that $S$ does not contain any elements from cluster $k$ is
\[
\mathbb{P}(\exists k, l(\bx) \not = k, \forall \bx \in S)\leq \sum_{k=1}^m (1-w_k)^n,
\]
which implies the result.
\end{proof}
\section{Proofs of Loss Bounds}
\label{sec:C}
In this section, the obtained concentration and separation results from Appendix \ref{sec:A} and Appendix \ref{sec:B} are used to obtain bounds on the loss function values.

First, it is proven that with high probability, the loss value of a negative is $-F$.
\begin{proposition}\label{prop2}
Given $N$ samples from a GMM with outliers, and $\sigma_{max}\leq\rho<\sqrt{0.6}$, then for a negative sample $\bx_j, l(\bx_j)=-1$, with probability at least $1-4N\exp\{ -p/128\}$, the loss satisfies
$L(\bx_j;\rho)=-F.$
\end{proposition}
\begin{proof}
From Corollary \ref{cor:sep_neg}, for $\bx_i, l(\bx_i)=-1$,  with probability at least $1-2\exp\{ -p/128\}$, the distance between a negative and $\bx_j$ satisfies
$$
\left\|\bx_{j}-\bx_{i}\right\|^2>1.5p, \quad l(\bx_i)=-1, i\not=j.
$$
Using the union bound, with probability at least $1-2N\exp\{ -p/128\}$, the distance between any other negative and $\bx_j$ satisfies
$$
\left\|\bx_{j}-\bx_{i}\right\|^2>1.5p,   \quad \forall \bx_i, l(\bx_i)=-1, i\not=j.
$$
Given $\sigma_{max}\leq\rho<\sqrt{0.6}$, then with probability at least $1-2N\exp\{ -p/128\}$,  the distance between any other negative and $\bx_j$ satisfies
$$\left\|\bx_{j}-\bx_{i}\right\|^2>1.5p>2.5p\rho^2, \quad \forall \bx_i, l(\bx_i)=-1, i\not=j.$$
From Corollary \ref{cor:sep_pos_neg}, for $\bx_i, l(\bx_i)>0$, with probability at least $1-2\exp\{ -p/128\}$, the distance between a positive and $\bx_j$ satisfies
$$
\left\|\bx_{j}-\bx_{i}\right\|^2>p(1.5+0.75\sigma_{l(\bx_i)}^2), \quad   l(\bx_i)>0. 
$$
Using the union bound, with probability at least $1-2N\exp\{ -p/128\}$, the distance between any positive and $\bx_j$ satisfies
$$
\left\|\bx_{j}-\bx_{i}\right\|^2>p(1.5+0.75\sigma_{l(\bx_i)}^2), \quad \forall \bx_i, l(\bx_i)>0.
$$
Given $\sigma_{max}\leq\rho < \sqrt{0.6}$, then with probability at least $1-2N\exp\{ -p/128\}$, the distance between any positive and $\bx_j$ satisfies
$$\left\|\bx_{j}-\bx_{i}\right\|^2>p(1.5+0.75\sigma_{l(\bx_i)}^2)>2.5p\rho^2  \quad \forall \bx_i, l(\bx_i)>0.$$
Therefore, with probability at least $1-4N\exp\{ -p/128\}$,  the distance between any other sample and $\bx_j$ satisfies
$$\left\|\bx_{j}-\bx_{i}\right\|^2>2.5p\rho^2, \quad \forall i\not=j.
$$
Therefore, with probability at least $1-4N\exp\{ -p/128\}$, it follows that
$$\ell(\left\|\bx_{j}-\bx_{i}\right\|; \rho)=\operatorname{min} \left(\frac{\left\|\bx_{j}-\bx_{i}\right\|^{2}}{p \rho^{2}}-2.5, 0\right) = 0, \forall i\not=j.
$$
Therefore, with probability at least $1-4N\exp\{ -p/128\}$, the loss satisfies
$$L(\boldsymbol{\bx_j}; \rho)=\sum\limits_{i=1}^{N} \ell\left(\left\|\bx_{j}-\bx_{i}\right\|; \rho\right)=-F,$$
since $\ell\left(\left\|\bx_{j}-\bx_{j}\right\|; \rho\right)=\ell(0;\rho)=-F$.
\end{proof}
Next, it is proven with high probability, the loss value of a positive is less than $-F.$
\begin{proposition}\label{prop3}
Given $N$ samples from a GMM with outliers, and $\sigma_{max}\leq\rho<\sqrt{0.6}$, then for a positive sample $\bx_j, l(\bx_j)=k>0$, with probability $1-2\exp\{ -p/128\}-\exp\{-(N-1)w_k\}$, the loss is bounded as $L(\bx_j;\rho)<-F.$
\end{proposition}
\begin{proof}
The probability that a sample of size $N-1$ contains no elements from cluster $S_k$ is 
\[
(1-w_k)^{N-1}\leq \exp\{-(N-1)w_k\}.
\] 
Therefore, with probability at least $1-\exp\{-(N-1)w_k\}$, there is at least one more sample $\bx_a, a\not=j$ besides $\bx_j$ in cluster $S_{k}$.\\
From Corollary \ref{cor:con_pos}, with probability at least $1-2\exp\{-p/128\}$, the distance between $\bx_a$ and $\bx_j$ is bounded as
$$
\left\|\bx_{j}-\bx_{a}\right\|^2<2.5 p\sigma_{k}^2.$$
Given  $\sigma_{max}\leq\rho<\sqrt{0.6}$, with probability at least $1-2\exp\{-p/128\}$, the distance between $\bx_a$ and $\bx_j$ is bounded as
$$
\left\|\bx_{j}-\bx_{a}\right\|^2<2.5 p\sigma_{k}^2 \leq2.5 p\sigma_{max}^2\leq 2.5 p\rho^2.
$$
Therefore, with probability at least $1-2\exp\{ -p/128\}-\exp\{-(N-1)w_k\}$, the following equality holds $$\ell(\left\|\bx_{j}-\bx_{a}\right\|; \rho)=\operatorname{min} \left(\frac{\left\|\bx_{j}-\bx_{a}\right\|^{2}}{p \rho^{2}}-2.5, 0\right) < 0.
$$
Therefore, with probability at least $1-2\exp\{ -p/128\}-\exp\{-(N-1)w_k\}$, the loss is bounded above as  $$L(\boldsymbol{\bx_j}; \rho)=\sum\limits_{i=1}^{N} \ell\left(\left\|\bx_{j}-\bx_{i}\right\|; \rho\right)\leq \ell\left(\left\|\bx_{j}-\bx_{a}\right\|; \rho\right)+\ell\left(\left\|\bx_{j}-\bx_{j}\right\|; \rho\right) <-F.$$
\end{proof}

\begin{proposition}\label{prop4}
Given $N$ samples from a GMM with outliers, with $w_i\geq a/m, i=\overline{1,m}$ for some $a>0$ and $\sigma_{max
}\leq\rho<\sqrt{0.6}$, randomly select a set $S$ of $|S|=n$ subsamples from it, then with probability at least $1-m\exp\{-na/m\}-2m\exp\{ -p/128\}-m\exp\{-a(N-1)/m\}$ 
for each $k=\overline{1,m}$ there exists $\bx_j\in S_k=\{\bx \in S,l(x)=k\}$ such that  
$L(\bx_j,\rho)<-F$.
\end{proposition}
\begin{proof}
According to Lemma \ref{lemma2}, the probability that a sample $S$ of size $n$ contains at least one observation from each cluster is
\[
1-\sum_{i=1}^m(1-w_i)^n \geq 1- m (1-a/m)^n \geq 1- m\exp\{-na/m\},
\]
and without loss of generality let $\bx_j$ be the observation from cluster $S_k, k=\overline{1,m}$.
Applying Proposition \ref{prop3} repeatedly to these $m$ samples and using the union bound, with probability at least $1-2m\exp\{ -p/128\}-\sum_{i=1}^m\exp\{-(N-1)w_i\}$, the loss is bounded above as $$L(\boldsymbol{\bx_j}, \rho)<-F.$$
Since $\forall w_i \geq a/m$, therefore $\sum_{i=1}^m\exp\{-(N-1)w_i\} \leq m\exp\{-a(N-1)/m\}$. 
Therefore, with probability at least $1-m\exp\{-na/m\}-2m\exp\{ -p/128\}-m\exp\{-a(N-1)/m\},$ for each $k=\overline{1,m}$ there exists $\bx_j\in S_k$ such that  
$$L(\bx_j,\rho)<-F.$$
\end{proof}
\section{Proofs of Theorem \ref{thm1} and Corollary \ref{cor7}}
\label{sec:D}
In this section, the proofs of Theorem \ref{thm1} and Corollary \ref{cor7} are given.
\begin{proof}\textbf{of Theorem} \ref{thm1}.
From Proposition \ref{prop2}, for a negative sample $\bx_j$, with probability at least $1-4N\exp\{ -p/128\}$, $L(\bx_j,\rho)=-F$, then for all the negatives, with probability at least $1-4N^2\exp\{-p/128\}$, the loss satisfies
$L(\bx_j,\rho)=-F.$\\
From Proposition \ref{prop4}, with probability at least $1-m\exp\{-na/m\}-2m\exp\{ -p/128\}-m\exp\{-a(N-1)/m\}$, for each $k=\overline{1,m}$ there is $\bx_j\in S_k, L(\bx_j,\rho)<-F$.\\
Combining  Proposition \ref{prop2} and Proposition \ref{prop4}, with probability at least $1-4N^2\exp\{ -p/128\}-m\exp\{-na/m\}-2m\exp\{ -p/128\}-m\exp\{-a(N-1)/m\}$, only positives will be selected at step 8 of SCRLM.\\
From Proposition \ref{prop1}, with probability at least $1-6N^2\exp\{-p/128\}$, all positives are correctly identified in Steps 9 and 17 and removed from negatives.\\
So with probability at least  $$1-10N^2\exp\{ -p/128\}-m\exp\{-na/m\}-2m\exp\{ -p/128\}-m\exp\{-a(N-1)/m\},$$
SCRLM will have $100\%$ accuracy. 
\end{proof}

\begin{proof} \textbf{of Corollary} \ref{cor7}.
The condition 
\[p >128  (2\log {N} +\log \frac{40}{\delta})\]
is equivalent to
\[10N^2\exp\{ -p/128\}<\frac{\delta}{4}.\]
The condition \[n>\frac{m}{a}(\log 4m-\log {\delta})\]
is equivalent to
\[m\exp(-na/m)<\frac{\delta}{4}.\]
The condition \[p >128  (\log 8m - \log {\delta})\]
is equivalent to
\[2m\exp\{ -p/128\}<\frac{\delta}{4}.\]
Finally, the condition \[
N>\frac{m}{a}(\log 4m-\log{\delta})+1.
\]
is equivalent to:
\[m\exp\{-a(N-1)/m\}<\frac{\delta}{4}.\]
These conditions together imply that
\[   1- 10N^2\exp\{ -p/128\}-m\exp(-na/m)-2m\exp\{ -p/128\}-m\exp\{-a(N-1)/m\}>1-\delta.
\]
According to Theorem \ref{thm1}, SCRLM has 100\% accuracy with probability at least $1-\delta$.
\end{proof}
\end{document}